\documentclass{article}
\pdfoutput=1

    \PassOptionsToPackage{numbers, compress}{natbib}

\usepackage[final]{neurips_2022}

\usepackage[utf8]{inputenc} 
\usepackage[T1]{fontenc}    
\usepackage{hyperref}       
\usepackage{url}            
\usepackage{booktabs}       
\usepackage{amsfonts}       
\usepackage{nicefrac}       
\usepackage{microtype}      
\usepackage{xcolor}         

\usepackage{amsmath}
\usepackage{amsthm}
\usepackage{amssymb}
\usepackage{enumitem}
\usepackage{array}
\usepackage{multirow}
\usepackage{wrapfig}
\usepackage{thmtools, thm-restate}
\usepackage{graphicx,subfig}
\usepackage{tabularx}

\usepackage{tikz,forest}
\usetikzlibrary{arrows.meta}

\usepackage{algorithm}
\usepackage[noend]{algorithmic}

\newcommand{\Tau}{\mathrm{T}}

\newcommand{\Bellman}{\mathcal{L}}

\newcommand{\M}{\mathcal{M}}
\newcommand{\B}{\mathcal{B}}

\newcommand{\X}{\mathcal{X}}
\newcommand{\A}{\mathcal{A}}
\newcommand{\SState}{\mathcal{S}}
\newcommand{\Cost}{\mathcal{C}}
\newcommand{\cmin}{c_{\min}}
\newcommand{\cmax}{c_{\max}}

\newcommand{\Prob}[1]{\mathbb{P}[#1]}

\newcommand{\pp}{\xleftarrow{+}}

\newcommand{\E}{\mathbb{E}}

\usepackage{color}

\newcommand*\xbar[1]{%
   \hbox{%
     \vbox{%
       \hrule height 0.5pt 
       \kern0.5ex
       \hbox{%
         \kern-0.1em
         \ensuremath{#1}%
         \kern-0.1em
       }%
     }%
   }%
} 
\graphicspath{{figures/}}

\theoremstyle{plain}
\newtheorem{thm}{Theorem}[section]

\newtheorem{lem}[thm]{Lemma}
\newtheorem{prop}[thm]{Proposition}
\newtheorem{cor}[thm]{Corollary}

\newtheorem{assumption}{Assumption}

\newtheorem{rem}[thm]{Remark}

\theoremstyle{definition}
\newtheorem{defn}{Definition}[section]

\newtheorem{prob}{Problem}[]

\theoremstyle{remark}

\newcommand*{\starnr}{\stepcounter{equation}\tag{\theequation}}


\newlength{\subcolumnwidth}
\newenvironment{subcolumns}[1][0.45\columnwidth]
 {\valign\bgroup\hsize=#1\setlength{\subcolumnwidth}{\hsize}\vfil##\vfil\cr}
 {\crcr\egroup}
\newcommand{\nextsubcolumn}[1][]{%
  \cr\noalign{\hfill}
  \if\relax\detokenize{#1}\relax\else\hsize=#1\setlength{\subcolumnwidth}{\hsize}\fi
}

\title{Policy Optimization with Linear Temporal Logic Constraints}

\author{%
  Cameron Voloshin \\
  Caltech \\
   \And
   Hoang M. Le \\
   Argo AI \\
   \And
   Swarat Chaudhuri \\
   UT Austin \\
   \And
   Yisong Yue \\
   Argo AI \\
   Caltech
}

\begin{document}

\maketitle


\begin{abstract}
We study the problem of policy optimization (PO) with linear temporal logic (LTL) constraints. The language of LTL allows flexible description of tasks that may be unnatural to encode as a scalar cost function. We consider LTL-constrained PO as a systematic framework, decoupling task specification from policy selection, and as an alternative to the standard of cost shaping. With access to a generative model, we develop a model-based approach that enjoys a sample complexity analysis for guaranteeing both task satisfaction and cost optimality (through a reduction to a reachability problem). Empirically, our algorithm can achieve strong  performance even in low-sample regimes.


\end{abstract}

\section{Introduction}\label{sec:intro}

The standard reinforcement learning (RL) framework aims to find a policy that minimizes a cost function. 
The premise is that this scalar cost function can completely capture the task specification (known as the ``reward hypothesis'' \cite{Sutton1998, silver2021reward}). To date, almost all theoretical understanding of RL is focused on this cost minimization setting (e.g., \cite{tsitsiklis1996analysis,kakade2002approximately,kakade2003sample,szepesvari2010algorithms,osband2013more,ayoub2020model,gheshlaghi2013minimax,dann2015sample,azizzadenesheli2016reinforcement,agarwal2020model,agarwal2020optimality,li2020breaking,qu2020finite,qu2021exploiting}).

However, capturing real-world task specifications using scalar costs can be challenging. For one, real-world tasks often consist of objectives that are required, as well as those that are merely desirable. By combining these objectives into scalar costs, one erases the distinction between these two categories of tasks. 
Also, there is recent theoretical evidence that certain tasks are simply not reducible to scalar costs \cite{Abel2021} (see Section \ref{sec:motivation}).  In practice, one circumvents these challenges using heuristics such as adding ``breadcrumbs'' \cite{Sorg2011}.
However, such heuristics can lead to catastrophic failures in which the learning agent ends up exploiting the cost function in an unanticipated way \cite{Randlv1998LearningTD,
toromanoff2019deep,
ibarz2018reward,
zhang2021importance,
ng1999policy}. 

In response to these limitations, recent work has studied alternative RL paradigms that use Linear Temporal Logic (LTL) to specify tasks (see Section \ref{sec:related_work}). LTL is a modeling language that can express desired characteristics of future paths of the system \cite{modelchecking}.
The notation is precise enough to allow the specification of both the required and desired behaviors; the cost minimization is left only to discriminate between which LTL-satisfying policy is ``best''. This ensures that the main objective --- e.g., time, energy, or effort --- does not have any relation to the task and is easily interpretable.

Existing work on RL with LTL constraints tends to make highly restrictive assumptions. Examples include (i) known mixing time of the optimal policy \cite{FuLTLPAC}, (ii) the assumption that every policy satisfies the task eventually \cite{Wolff2012RobustControl}, or (iii) known optimal discount factor \cite{Hasanbeig2018lcrl}, all of which assist in task satisfaction verification. These assumptions have complex interactions with the environment, making them impractical if not impossible to calculate.
The situation is made more complex by recent theoretical results \cite{Yang2021Intractable,alur2021framework} that show that there are LTL tasks that are not PAC-MDP-learnable.

In this paper, we address these limitations through a novel policy optimization framework for RL under LTL constraints. Our approach relies on two assumptions that are significantly less restrictive than those in prior work and circumvent the negative results on RL-modulo-LTL: the availability of a generative model of the environment and a lower bound on the transition probabilities in the underlying MDP. Under these assumptions, we derive a learning algorithm based on a reduction to a reachability problem. 
The reduction in our method can be instantiated with several planning procedures that handle unknown dynamics \cite{bertsekas2011dynamic, puterman2014markov}. 
We show that our algorithm offers strong constraint satisfaction guarantees and give a rigorous sample complexity analysis of the algorithm. 

In summary, the contributions of this paper are:
\setlength{\leftmargini}{10pt}
\begin{enumerate}
\item We provide a novel approach to LTL-constrained RL that requires significantly fewer assumptions, and offers stronger guarantees, than previous work.

\item We develop several new theoretical tools for our analysis. These 
may be of independent interest.
\item We empirically validate using both infinite- and indefinite-horizon problems, and with composite specifications such as collecting items while avoiding enemies.  We find that our method enjoys strong performance, often requiring many fewer samples than our worst-case guarantees.
\end{enumerate}

\section{Motivating Examples} \label{sec:motivation}

We examine two examples where standard cost engineering cannot capture the task (Figure \ref{fig:motivation}). We consider the undiscounted setting here. See \cite{Littman2017, Abel2021} for difficult examples for the discounted setting. 

\textbf{Example 1 (Infinite Loop).} A robot is given the task of perpetually walking between the coffee room and the office (Figure \ref{fig:motivation} (Left)). 
To achieve this behavior, both the policy and cost-function must be history-dependent. These can be made Markovian through proper state-space augmentation and has been studied in  hierarchical reinforcement learning or learning with options \cite{Le2018, Sutton1999Option}. Options engineering is laborious and requires expertise. Nevertheless, without the appropriate augmentation, any cost-optimal policy of a Markovian cost function will fail at the task. We will see in Section \ref{sec:problem_formulation} that any LTL expression comes with automatic state-space augmentation, requiring no expert input.

\textbf{Example 2 (Safe Delivery).} 
The goal is to maximize the probability of safely sending a packet from one computer to another (Figure \ref{fig:motivation} (Right)).  
Policy $1$ leads to a hacker sniffing packets but passing them through, and is unsafe. Policy $2$ leads to a hacker stealing packets with probability $p > 0$, and is safe with probability $1-p$, and is the policy that satisfies the task. 
For cost engineering, let $R$ and $S$ be the recurring costs of the received and stolen states. For the two policies, the avg.~costs are $g_1 = R$ and $g_2 = pS + (1-p)R$.
Strangely, we must set $R>S$ in order for $g_2 < g_1$.
Fortunately, optimizing any cost function constrained to satisfying the LTL specification does not suffer from this counter intuitive behavior as only policy $2$ has any chance of satisfying the LTL expression.

\begin{figure}[h]
\begin{center}
\begin{tabular}{cc}
\minipage{.45\textwidth}
  \begin{center}
  \includegraphics[width=0.7\linewidth]{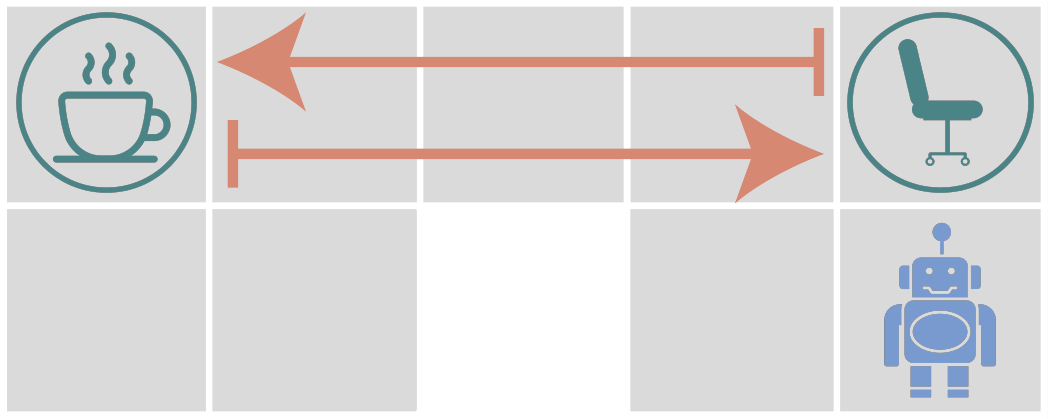}
  \end{center}
\endminipage
&
\minipage{.45\textwidth}
  \begin{center}
  \includegraphics[width=0.7\linewidth]{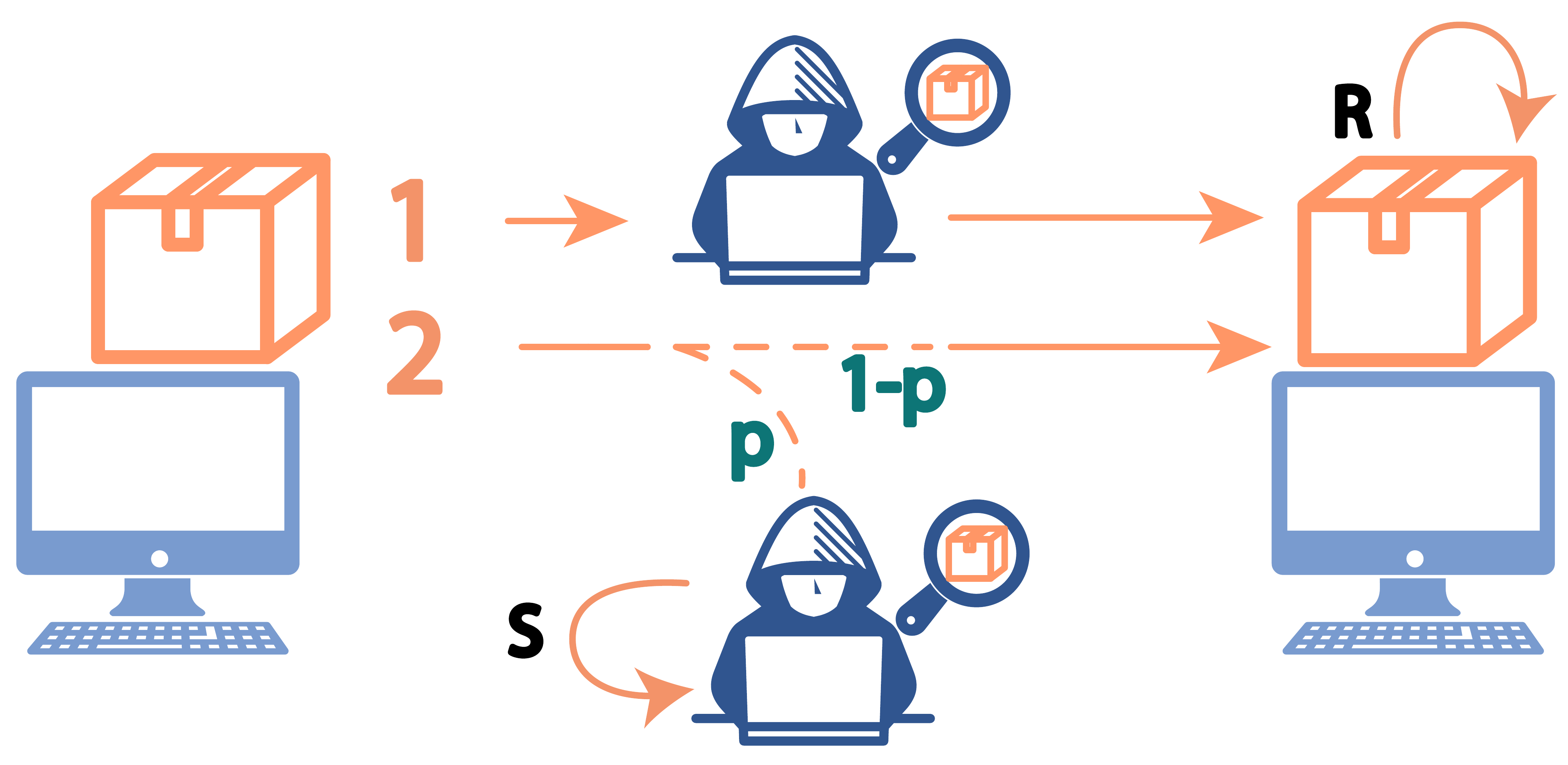}
  \end{center}
\endminipage
\end{tabular}
\end{center}
\caption{\textit{ (Left) Infinite Loop. The robot must perpetually walk between the coffee room and office. Without proper state-space augmentation, a markovian cost function cannot capture this task. (Right) Safe Delivery. The specification is to deliver a packet without being interfered. Policy 2 should be chosen. One would need to penalize receiving the packet significantly over having it stolen: $R > S$.
}}
 \label{fig:motivation}
\end{figure} 
\section{Background and Problem Formulation}\label{sec:problem_formulation}

We now formulate the problem.
An \emph{atomic proposition} is a variable that takes on a truth value. An \emph{alphabet} over a set of atomic propositions $\text{AP}$ is given by 
$\Sigma = 2^{\text{AP}}$. For example, if $\text{AP} = \{a, b\}$ then $\Sigma = \{\{\}, \{a\}, \{b\}, \{a,b\}\}$. $\Delta(X)$ represents the set of probability distributions over a set $X$. 

\subsection{MDPs with Labelled State Spaces} 
\label{sec:mdp}

We assume that the environment follows the finite Markov Decision Process (MDP) framework given by the tuple $\M = (\SState^\M, \A^\M, P^\M, \Cost^\M, d_0^\M, L^{\mathcal{M}})$ consisting of a finite state space $\SState^\M$, a finite action space $\A^\M$, an \textit{\textbf{unknown}} transition function $P^\M : \SState^\M \times \A^\M \to \Delta(\SState^\M)$, a cost function $\Cost : \SState^\M \times \A^\M \to \Delta([\cmin, \cmax])$, an initial state distribution $d_0 \in \Delta(\SState^\M)$, and a labelling function $L^{\mathcal{M}}: \SState^\M \to \Sigma$. We take $\A^\M(s)$ to be the set of available actions in state $s$. Unlike traditional MDPs, $\M$ has a labeling function $L^{\mathcal{M}}$ which returns the atomic propositions that are true in that state. A \textbf{run} in $\M$ is a sequence of 
states 
$\tau = (s_0,s_1,\ldots)$ reached through successive transitions.

\subsection{Linear Temporal Logic (LTL), Synchronization with MDPs, and Satisfaction} 

Now we give some basic background on LTL. For a more comprehensive overview, see \cite{modelchecking}.

\begin{defn}[LTL Specification, $\varphi$]
An LTL specification $\varphi$ is the entire description of the task, including both desired and required behaviors, and is constructed from a composition of atomic propositions, including logical connectives: not ($\lnot$), and ($\&$), and implies
($\rightarrow$); and temporal operators: next $(X)$, repeatedly/always/globally ($G$), eventually ($F$), and until ($U$).
\end{defn}

\textbf{Examples.} Consider again the examples in Section \ref{sec:motivation}.
For $AP = \{a,b\}$, some basic task specifications include safety ($G \lnot a$), reachability ($F a $),
stability ($FG a$), response ($a \rightarrow F b$), and progress $(a \;\&\; X F b)$. For the Infinite Loop example (Figure \ref{fig:motivation} (Left)), $AP = \{o, c\}$ indicating the label of the grid location of our agent (office, coffee, or neither). The specification is ``$GF(o \;\&\; XF c)$'' meaning ``go between office and coffee forever'', and is a combination of safety, reachability, and progress. For the Safe Delivery example (Figure \ref{fig:motivation} (Right)), $AP = \{s\}$ indicating the safety of a state. The specification is ``$G s$'' meaning ``always be safe''. 


\textbf{LTL Satisfaction: Synchronizing MDP with LTL.}
By synchronizing an MDP with an LTL formula, we can easily check if a run in the MDP satisfies a specification $\varphi$. In particular, it is possible to model the progression of satisfying $\varphi$ through a specialized automaton, an LDBA $\mathcal{B}_\varphi$ \cite{Sickert2016ldba}, defined below. More details for constructing LDBAs are in \cite{hahn2013lazy, modelchecking, kvretinsky2018owl}. We drop $\varphi$ from $\B_\varphi$ for  brevity. 

\begin{defn}(Limit Deterministic B\"uchi Automaton, LDBA \cite{Sickert2016ldba})
\label{def:ldba}
An \textbf{\textit{LDBA}} is a tuple $\mathcal{B} = (\SState^\B, \Sigma \cup \A_\B, P^\B, \SState^\mathcal{B \ast}, s^\B_0)$ consisting of (i) a finite set of states $\SState^\B$, (ii) a finite alphabet $\Sigma = 2^{\text{AP}}$, $\A_\B$ is a set of indexed jump transitions (iii) a transition function $P^\B : \SState^\B \times (\Sigma \cup \A_\B) \to 2^{\SState^\B}$, (iv) accepting states $\SState^\mathcal{B \ast} \subseteq \SState^\B$, and (v) initial state $s^\B_0$. There exists a mutually exclusive partitioning of $\SState^\B = \SState^\B_D \cup \SState^\B_{N}$ such that $\SState^\mathcal{B \ast} \subseteq \SState^\B_D$, and for $s \in S^\B_D, a \in \Sigma$ then $P^\B(s, a) \subseteq \SState^\B_D$ and $|P^\B(s, a)| = 1$, deterministic. $\A_\B(s)$ is only (possibly) non-empty for $s \in \SState^\B_{D}$ and allows $\B$ to transition without reading an AP. 
A \textit{path} $\sigma = (s_0, s_1, \ldots)$ is a sequence of states in $\mathcal{B}$ reached through successive transitions.  $\mathcal{B}$ \textbf{accepts} a path $\sigma$ if there exists some state $s \in \SState^\mathcal{B \ast}$ in the path that is visited infinitely often.
\end{defn}

We can now construct a synchronized product MDP from the interaction of $\M$ and $\B$.

\begin{defn}(Product MDP)\label{def:product_mdp}
The product MDP $\mathcal{X}_{\M, \B} = (\SState, \A, P, \Cost, d_0, L, \SState^{\ast})$ is an MDP with $\SState = \SState^\M \times \SState^\B$, $\A = \A^{\M} \cup \A^{\B}$, $\Cost((m, b),a) = \Cost^{\M}(m,a)$ if $a \in A^{\M}(m)$ otherwise $0$, $d_0 = \{(m,b)| m \in d_0^\M, b \in P^\mathcal{B}(s_0^\mathcal{B}, L^\M(m))\}$, $L((m,b)) = L^\M(m)$, $S^\ast = \{(\cdot, b) \in \SState | b \in \SState^{\B \ast}\}$ accepting states, and $P: \SState \times \A \to \Delta(\SState)$ taking the form:
\begin{equation*}
    P((m, b),a,(m', b')) = 
    \begin{cases} 
    P^{\M}(m,a,m') 
    &a \in A^{\M}(m), b' \in P^\B(b, L(m')) \\
    1, &a \in  A^{\B}(b), b' \in P^\B(b, a), m=m' \\
    0, &\text{otherwise}
    \end{cases}
\end{equation*} 
\end{defn}

A run $\tau = (s_0, s_1, \ldots) = ((m_0, b_0), (m_1, b_1), \ldots)$ in $\X$ is accepting (accepted) if $(b_0,b_1,\ldots)$, the projection onto $\B$, is accepted. Equivalently, some $s \in \SState^\ast$ in $\X$ is visited infinitely often. This leads us to the following definition of LTL satisfaction:

\begin{defn}[Satisfaction, $\tau \models \varphi$]
A run $\tau$ in $\X$ \textit{satisfies} $\varphi$, denoted $\tau \models \varphi$, if it is accepted.
\end{defn}

\begin{defn}(Satisfaction, $\pi \models \varphi$)
A policy $\pi \in \Pi$ \textit{satisfies} $\varphi$ with probability $\Prob{\pi \models \varphi} = \E_{\tau \sim \Tau^P_\pi}[\mathbf{1}_{\tau \models \varphi}]$. Here, $\mathbf{1}_{X}$ is an indicator variable which is $1$ when $X$ is true, otherwise $0$. $\Tau^P_{\pi}$ is the set of trajectories induced by $\pi$ in $\X$ with transition function $P$.
\end{defn}

\subsection{Problem Formulation}
\label{sec:problem}


Our goal is to find a policy that simultaneously satisfies a given LTL specification $\varphi$ with highest probability (probability-optimal) and is also optimal w.r.t.~the cost function of the MDP. 
We consider (stochastic) Markovian policies $\Pi$, and define the set of all probability-optimal policies as $\Pi_{\max} = \left\{ \arg\max_{\pi' \in \Pi} \Prob{\pi' \models \varphi}\right\}$.  
We first define the gain $g$ (average-cost) and transient cost $J$:
\begin{equation}\label{def:costs}
\begin{split}
g^P_\pi \equiv &\mathbb{E}_{\tau \sim \Tau^P_{\pi}} \left[ \lim_{T \to \infty} \frac{1}{T} \sum_{t=0}^{T-1} \Cost(s_t, \pi(s_t)) \bigg| \tau \models \varphi \right], \  J^P_\pi \equiv \mathbb{E}_{\tau \sim \Tau^P_{\pi}}\left[\sum_{t=0}^{\kappa_{\tau}} \Cost(s_t, \pi(s_t))  \bigg| \tau \models \varphi \right]
\end{split}
\end{equation}
where $\kappa_\tau$ is the first (hitting) time the trajectory $\tau$ leaves the transient states induced by $\pi$. When $P$ is clear from context, we abbreviate $g^P_\pi$ and $J^P_\pi$ by $g_\pi$ and $J_\pi$, respectively.

Gain optimality for infinite horizon problems has a long history in RL \cite{bertsekas2011dynamic, puterman2014markov}. Complementary to gain optimality, we consider a hybrid objective including the transient cost. For any $\lambda \geq 0$, we define the optimal policy as the probability-optimal policy with minimum combined cost:
\begin{align*}
\tag{OPT}\label{eqn:main_problem}
\pi^\ast_{\lambda} \equiv \arg\min_{\pi \in \Pi_{\max}} \; &J_\pi + \lambda g_\pi  = \arg\min_{\pi \in \Pi_{\max}} \; (J_\pi + \lambda g_\pi )  \Prob{\pi \models \varphi} \quad (\equiv V^P_{\pi, \lambda}).
\end{align*}
In other words, probability-optimal policies are those that satisfy the entirety of the task, both desired and required behaviors, where $V^P_{\pi, \lambda} \equiv (J_\pi + \lambda g_\pi )  \Prob{\pi \models \varphi}$ is the normalized value function\footnote{Normalized objectives are not unusual in RL, e.g. in discounted settings, multiplication by $(1-\gamma)$}, corresponding to a notion of energy or effort required, with $\lambda$ representing the tradeoff between gain and transient cost.
We will often omit the dependence of $V$ on $P$ and $\lambda$ for brevity.

\textbf{Example.} Consider the Safe Delivery example (Figure \ref{fig:motivation} (Right)). For policy $1$, $\Prob{1 \models \varphi} = 0$ and so $1 \not\in \Pi_{\max}$. Let policy $2$ be a cost $1$ timestep before stolen or receipt, then $g_{2} = R$ is the (conditional) gain, $J_{2} = 1$ is the (conditional) transient costs, $\Prob{2 \models \varphi} = 1-p$, and $V_{2} = (1+\lambda R)(1-p).$

\begin{prob}[Planning with Generative Model/Simulator]
\label{prob:main}
Suppose access to a generative model of the true dynamics $P$ from which we can sample transitions $s' \sim P(s,a)$ for any state-action pair $(s,a) \in \SState \times \A$.\footnote{The use of a generative model is increasingly common in RL  \cite{gheshlaghi2013minimax,li2020breaking,agarwal2020model,Tarbouriech2021}, and is applicable in many settings where such a generative model is readily available as a simulator (e.g., \cite{dosovitskiy2017carla}).}  With probability $1-\delta$, for some errors $\epsilon_{\varphi}, \epsilon_{V} >0$, find a policy $\pi \in \Pi$ that simultaneously has the following properties: $(i) \;\; |\Prob{\pi \models \varphi} - \Prob{\pi^\ast \models \varphi}| < \epsilon_{\varphi} \quad (ii)\;\; |V_{\pi} - V_{\pi^\ast}| < \epsilon_{V}$.
\end{prob}

\section{Approach}

\subsection{End Components \& Accepting Maximal End Components}

Our analysis relies on the idea of an end component: a recurrent, inescapable set of states when restricted to a certain action set. It is a sub-MDP of a larger MDP that is probabilistically closed.

\begin{defn}(End Component, EC/MEC/AMEC \cite{modelchecking})\label{def:end_component} Consider MDP $(\SState, \A, P, \mathcal{C}, d_0, L, \SState^\ast)$. An end component $(E, \A_E)$ is a set of states $E \subseteq \SState$ and acceptable actions $\A_E(s) \subseteq \A(s)$ (where $s \in E$) such that $\forall(s,a) \in E \times \A_E$ then $Post(s,a) = \{s' |  P(s,a, s') > 0 \} \subseteq E$. Furthermore, $(E, \A_E)$ is strongly connected:  any two states in $E$ is reachable from one another by means of actions in $\A_E$. We say an end component $(E, \A_E)$ is \textit{maximal} (MEC) if it is not contained within a larger end component $(E', \A_{E'})$, ie. $\nexists (E', \A_{E'})$ EC where $E \subseteq E', \A_E(s) \subseteq \A_{E'}(s)$ for each $s \in A$. A MEC $(E, \A_E)$ is an \textit{accepting} MEC (AMEC) if it contains an accepting state,  $\exists s \in E$ s.t. $s \in \SState^\ast$.
\end{defn}

\subsection{High-Level Intuition}\label{subsec:intuition}

The description of our approach, LTL Constrained Planning (LCP), in Section \ref{sec:algo} is rather technical in order to yield theoretical guarantees.  We thus first summarize the high-level intuitions.

\begin{figure}[!htp]
\centering

\begin{minipage}{1\textwidth}
\begin{subcolumns}[0.33\textwidth]
  \subfloat[Abstract Diagram]{\includegraphics[width=\subcolumnwidth]{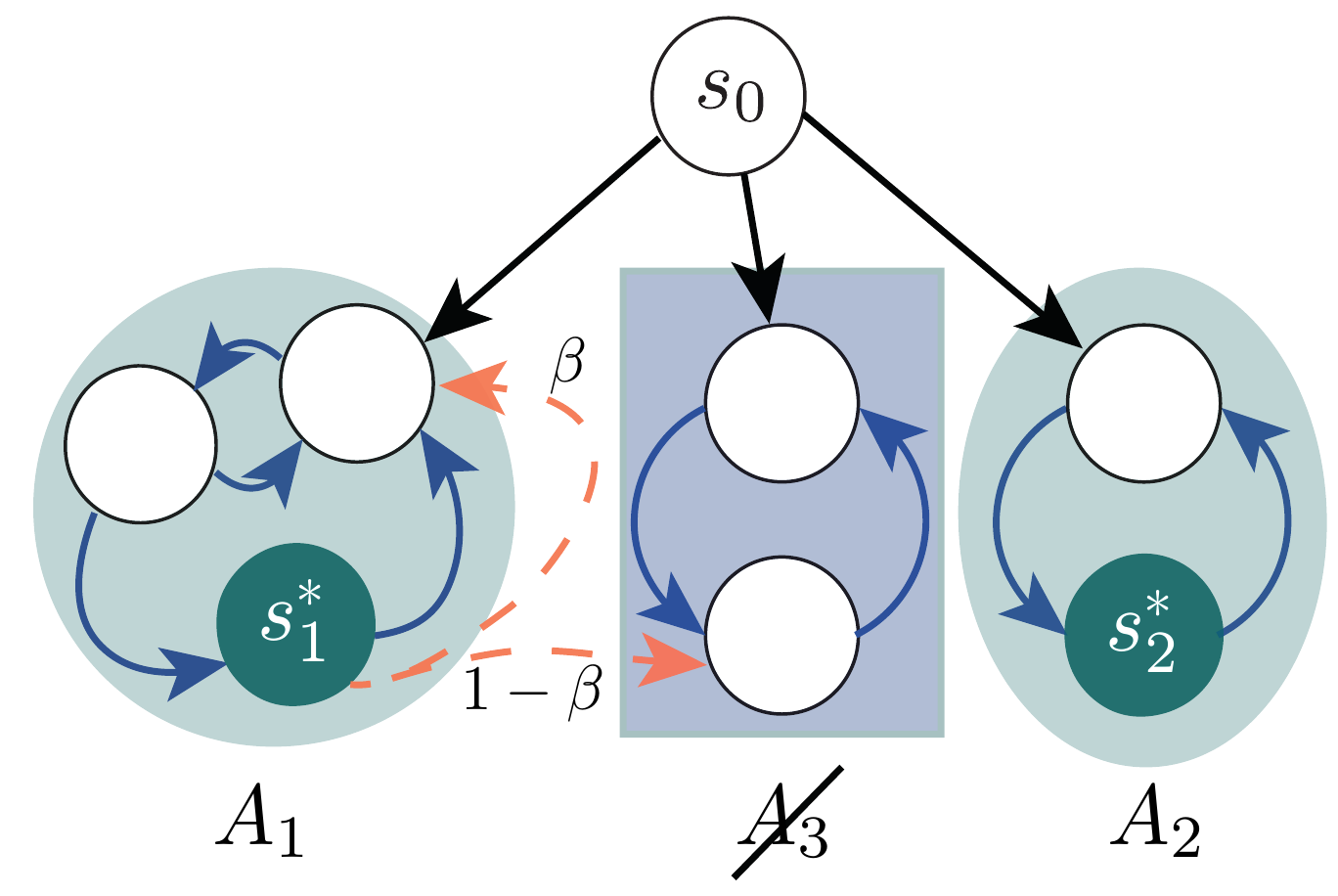}} 
\nextsubcolumn[0.33\textwidth]
  \subfloat[Example, Infinite Loop ]{\includegraphics[width=\subcolumnwidth]{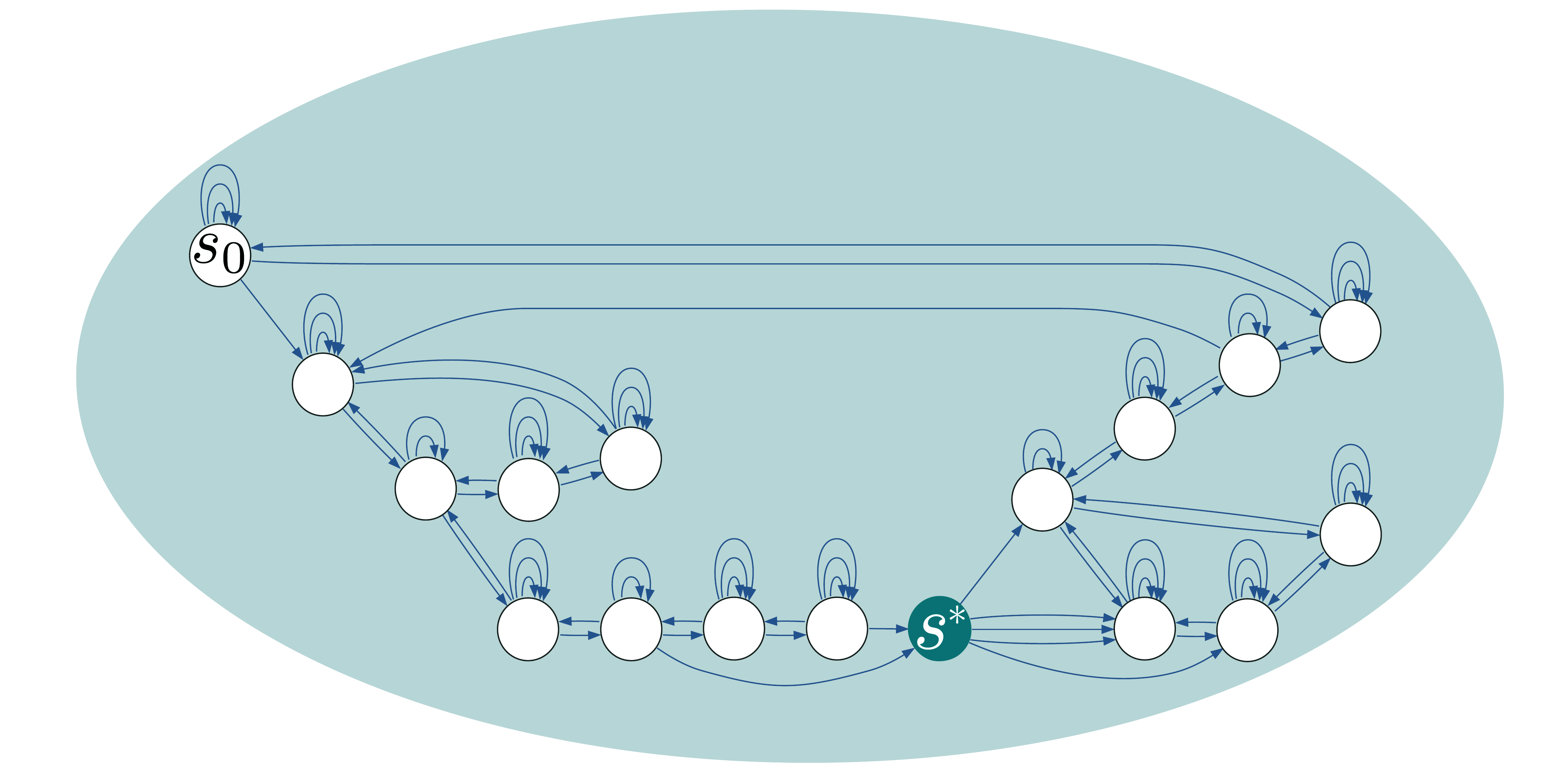}}
\nextsubcolumn[.33\textwidth]
  \subfloat[Example, Safe Delivery]{\includegraphics[width= \subcolumnwidth]{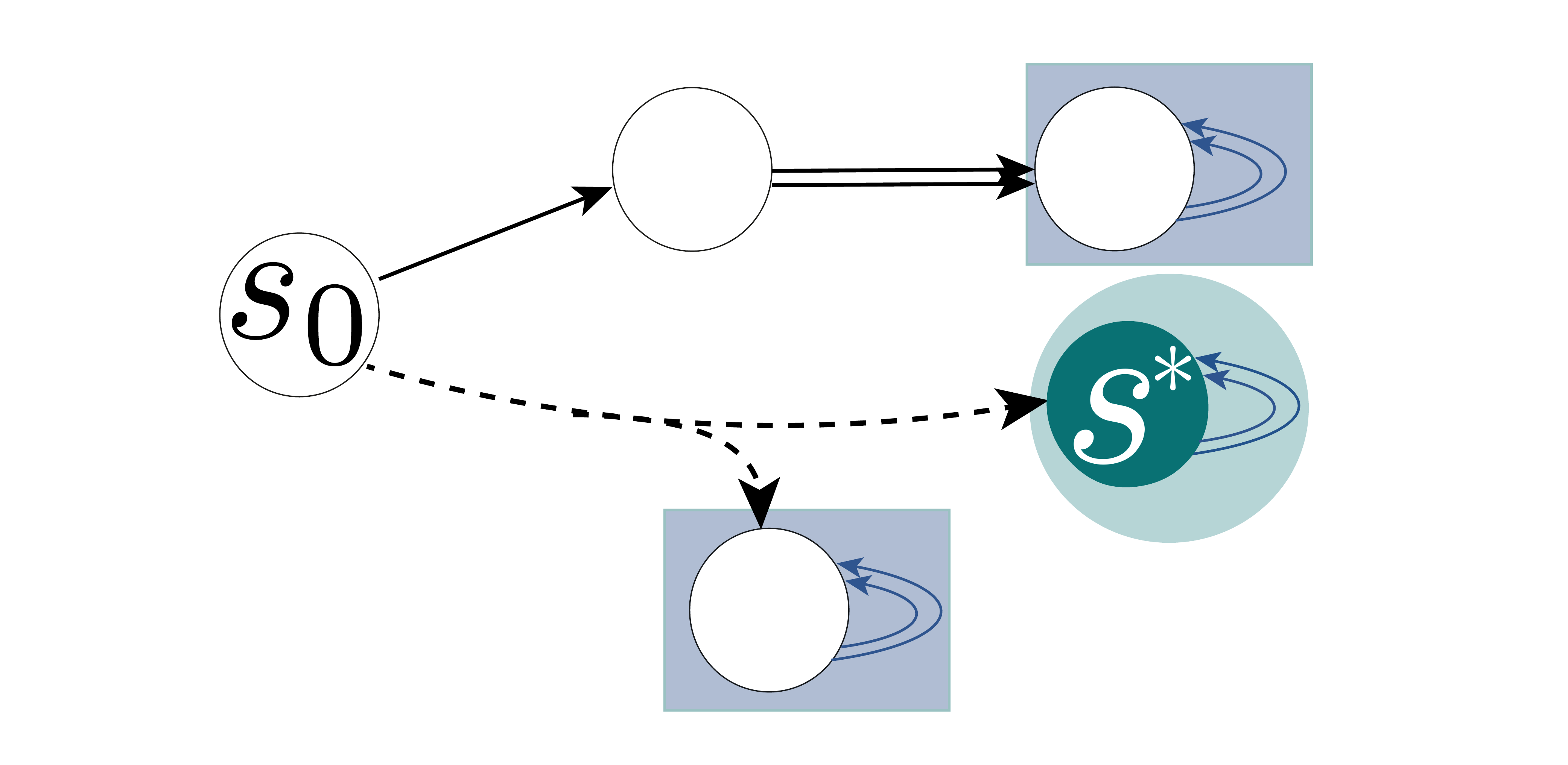}}
\end{subcolumns}
\end{minipage}
\caption{Product MDP diagrams. (Left) The goal of LTL Constrained Policy Optimization can be reduced to a reachability problem. We want to reach $A_1$ or $A_2$ from $s_0$ and then follow the blue arrows with some distribution. $A_3$ with the blue arrows is a rejecting end component because it does not contain an accepting state $s^\ast$. For $\beta < 1$ , the yellow action is not in the allowable action set of $A_1$ because there is a risk of entering $A_3$, strictly decreasing our probability of LTL satisfaction. (Center) Example for Infinite Loop, Figure \ref{fig:motivation} Left. (Right) Example for Safe Delivery, Figure \ref{fig:motivation} Right.}
\label{fig:intuition}
\end{figure}

\textbf{Solution Decomposition.} 
Consider the accepting states $s_1^\ast, s_2^\ast$ in Figure \ref{fig:intuition} (Left), which are the states we need to visit infinitely often to satisfy the specification.
First, let us identify the accepting maximal end components (AMECs) of $s_1^\ast$ and $s_2^\ast$: the state sets $A_1$ and $A_2$ (resp.) and their corresponding action sets $\A_{A_1}$ and $\A_{A_2}$ (the blue arrows in $A_1$ and $A_2$). Note that these AMECs do not include the yellow action in Figure \ref{fig:intuition} (Left), which has a chance of leaving $A_1$ and getting stuck in $A_3$.  

Our solution first runs a \textit{transient} policy until reaching $A_1$ or $A_2$, and then switches to a (probability-optimal) \textit{recurrent} policy that stays within $A_1$ or $A_2$ (resp.) while visiting $s_1^\ast$ or $s_2^\ast$ (resp.) infinitely often.
A probability-optimal \textit{recurrent} policy will select actions in $\A_{A_1}$ and $\A_{A_2}$ to visit $s_1^\ast, s_2^\ast$ infinitely often (e.g., the uniform policies with the AMECs $(A_1, \A_{A_1})$ and $(A_2, \A_{A_2})$).
Finding a \textit{transient} policy from $s_0$ to $A_1, A_2$ can be viewed as a reachability problem, which we can solve via a Stochastic Shortest Path (SSP) problem and leverage recent literature \cite{Tarbouriech2021, Kolobov2012DeadEnd}.


\textbf{Cost Optimality.} As stated in \ref{eqn:main_problem}, the goal is to find a cost-optimal policy within the set of probability-optimal policies.
For instance, the uniform policy over  $\A_{A_1}$ and $\A_{A_2}$ (the blue arrows in Figure \ref{fig:intuition} (Left) is probability optimal, but may not be cost optimal. Similarly, the unconstrained cost-optimal policy may not be probability optimal. Consider just $A_1$ for the moment. Suppose the cost of the arrows between the white nodes is $4$ while the other costs are $7$. Then the uniform (probability-optimal) policy in $A_1$ over $\A_{A_1}$ has cost $\frac{1}{2}\left(\frac{4 + 4}{2}\right) + \frac{1}{2}\left(\frac{7 + 7 + 4}{3}\right) = 5$. The gain-optimal policy that deterministically selects the actions between the white nodes $\tilde{\pi}$ has cost $\left(\frac{4 + 4}{2}\right) = 4$, but is not probability optimal. If we perturb $\tilde{\pi}$ to make it even slightly stochastic (but still mostly deterministic, i.e $\eta$-greedy with $\eta \approx 0$), then it will be arbitrarily close to gain optimality and also recover probability optimality. This is a preferable probability-optimal policy over the uniform policy. 

\textbf{Overall Procedure.} The high-level procedure is: (i) identify the AMECs (e.g. $(A_1, \A_{A_1}), (A_2, \A_2))$ by filtering out bad actions like the yellow arrow; (ii) find a cost-optimal (optimal gain cost) recurrent policy in each AMEC that visits some $s^\ast$ infinitely often; (iii) instantiate an SSP problem that finds a cost-optimal (optimal transient cost) transient policy from $s_0$ to $A_1 \cup A_2$ and avoids $A_3$; (iv) return a policy that stitches together the policies from $(ii)$ and $(iii)$.
See Section \ref{sec:algo}  for the algorithmic details. We show in Section \ref{sec:analysis} that this solution gives the optimal solution to \ref{eqn:main_problem}.

\subsection{Additional Assumptions and Definitions}\label{subsec:assumptions} 



Perhaps surprisingly, when planning with a simulator (i.e., generative model), even infinite data is insufficient to verify an LTL formula without having a known lower-bound on the lowest nonzero probability of the transition function $P$ \cite{Littman2017}. Without this assumption, LTL constrained policy learning is not learnable \cite{Yang2021Intractable}.  We thus begin by assuming a known lower bound on entries in $P$.\footnote{Our assumptions are consistent with the minimal requirements studied by \cite{Littman2017}}
%
%
\begin{assumption}[Lower Bound]\label{assump:lower_bound}
We assume we have access to a lower bound $\beta > 0$ on the lowest non-zero probability of the transition function $P$ (Sec.~\ref{sec:mdp}):
\begin{equation}
    0 < \beta \leq \min_{s,a,s' \in \SState\times\A\times\SState} \{P(s,a,s') | P(s,a,s') > 0\}.
\end{equation}
\end{assumption}

We assume that all the costs are strictly positive, avoiding zero-cost (or negative-cost) cycles that trap a policy. Leveraging cost-perturbations and prior work \cite{Tarbouriech2021} can remove the assumption.
\begin{assumption}[Bounds on cost function]\label{assump:cost_function}
The minimum cost $c_{\min} > 0$ (Sec.~\ref{sec:mdp}) is strictly positive.
\end{assumption}




Let $D = \{(s,a,s')\}$ be  all the collected samples $(s,a,s')$ while running the algorithm. At any point, $\widehat{P}(s,a,s') = \frac{|\{(s,a,s') \in D\}|}{|\{(s,a) \in D\}|}$ is the empirical frequency of visiting $s'$ from $(s,a)$. 
We  introduce an event $\mathcal{E}$ and error $\psi(n)$ to quantify uncertainty on $\widehat{P}(s,a,s')$ based on current data: $n(s,a) =  |\{(s,a) \in D\}|$. $\mathcal{E}$ is based on empirical Bernstein bounds \cite{Maurer2009}, and holds w.p.~$1-\delta$ (Lemma \ref{lem:event}).

\begin{restatable}[High Probability Event]{defn}{highProbEvent}
\label{def:high_prob_event} A high probability event $\mathcal{E}$:
$$\mathcal{E} = \{\forall s,a, s' \in S \times A \times S, \forall n(s,a) > 1: | (P(s,a,s') - \widehat{P}(s,a,s')) | \leq  \psi_{sas'}(n) \leq \psi(n) \},$$
where $\psi_{sas'}(n) \equiv \sqrt{2 \hat{P}(s,a,s')(1-\hat{P}(s,a,s')))\xi(n)} + \frac{7}{3}\xi(n)$, $\psi(n) \equiv  \sqrt{\frac{1}{2} \xi(n)} + \frac{7}{3}\xi(n)$, and $\xi(n) \equiv \log(\frac{4 n^2 |\SState|^2 |\A|}{\delta})/(n-1)$.
\end{restatable}

\begin{rem}
For some $\rho > 0$, if we require $|P(s,a,s') - \widehat{P}(s,a,s')| \leq \rho$ then we need $n(s,a) = \psi^{-1}(\rho)$ samples for state-action pair $(s,a)$. See Lemma \ref{lem:samples_req} for the quantity $\psi^{-1}(\rho)$. 
\end{rem}




\begin{restatable}[Plausible Transition Function]{defn}{ptf}
\label{def:plausible_transition} The set of plausible transition functions is given by
\begin{small}
\begin{equation}
    \mathcal{P} = \{ \tilde{P}: \SState \times \A \to \Delta(\SState) | \begin{cases}
    \tilde{P}(s,a,s') = \widehat{P}(s,a,s'), \quad &\widehat{P}(s,a,s') \in \{0,1\}\\
    \tilde{P}(s,a,s') \in \widehat{P}(s,a,s') \pm \psi_{sas'} \cap [\beta, 1-\beta], \quad &\text{otherwise}
    \end{cases}\}
\end{equation} 
\end{small}
\end{restatable}

Let $\mathcal{P}(s,a) \equiv \{P(s,a,\cdot) | P \in \mathcal{P}\}$ be the possible transition distributions for state-action pair $(s,a)$. We denote $P_\pi(s,s') = \mathbb{E}_{a \sim \pi}[P(s,a,s')]$ as the Markov chain given dynamics $P$ with policy $\pi$, and can be thought of as a $|\SState| \times |\SState|$ matrix $P_\pi = \{p_{ij}\}_{i,j=1}^{|\SState|}$.

\subsection{Main Algorithm: LTL Constrained Planning (LCP)}
\label{sec:algo}

\begin{minipage}{1\textwidth}
\begin{algorithm}[H]
	\begin{small}
	\caption{ LTL Constrained Planning (LCP)} 
	\label{algo:main}
	\begin{algorithmic}[1]
	    \REQUIRE Error $\epsilon_V > 0$, Error $\epsilon_{\varphi} > 0$, Tolerance $\delta > 0$, Lower bound $\beta > 0$ (see Assumption \ref{assump:lower_bound})
	    \STATE Globally, track $\widehat{P}(s,a,s') = \frac{|\{(s,a,s') \in D\}|}{|\{(s,a) \in D\}|}$ \hfill$\COMMENT{\text{Empirical estimate of } P}$
	    \STATE $((A_1, \A_{A_1}), \ldots, (A_m, \A_{A_k})) \leftarrow \texttt{FindAMEC}((\SState, \A, \widehat{P}))$
	    \FOR{$i = 1, \ldots, k$}
	       \STATE Set $\pi_i, g_i \leftarrow \texttt{PlanRecurrent}((A_i, \A_{A_i}), \frac{\epsilon_V}{7 \lambda})$ \hfill$\COMMENT{\text{Plan gain-optimal policy } \pi_i \text{ for } A_i}$
	    \ENDFOR
	    \STATE Set $\pi_0 \leftarrow \texttt{PlanTransient}(((A_1, g_1), \ldots,(A_k, g_k)), \frac{2\epsilon_V}{9})$ \hfill$\COMMENT{\text{Plan shortest paths policy } \pi_0 \text{ to } \cup_{i=1}^k A_i}$
	\RETURN $\pi = \cup_{i=0}^k \pi_i$
	\end{algorithmic}
	\end{small}
\end{algorithm}
\end{minipage}

Our approach, LTL Constrained Planning (LCP), has three  components, as shown in Algorithm \ref{algo:main} and described below.  Recall from Problem \ref{prob:main} that the policy optimization problem \ref{eqn:main_problem} is instantiated over a product MDP (Def.~\ref{def:product_mdp}), and that we are given a generative model of the true dynamics $P$ from which we can sample transitions $s'\sim P(s,a)$ for any state/action pair.

\textbf{Finding AMECs (\texttt{FindAMEC}).} After sampling each state-action pair $\phi_{\texttt{FindAMEC}} = O(\frac{1}{\beta})$ times (see Prop.~\ref{prop:support_verification}), by Assumption \ref{assump:lower_bound}, we can verify the support of $P$. We can compute all of the MECs using Algorithm 47 from \cite{modelchecking}. Among these MECs, we keep the AMECs, which amounts to checking if the MEC $(A_i, \A_{A_i})$ contains an accepting state $s^\ast\in \SState^\ast$ from the given product MDP.

\textbf{PlanRecurrent ($\texttt{PR}$).} 
To plan in each AMEC $(A, \A_A)$ (i.e., find the optimal recurrent 
\begin{wrapfigure}{l}{0.49\textwidth}
\vspace{-0.8cm}
\begin{minipage}{.48\textwidth}
\begin{algorithm}[H]
	\begin{small}
	\caption{$\texttt{PlanRecurrent}$ ($\texttt{PR}$) } 
	\label{algo:plan_recurrent}
	\begin{algorithmic}[1]
	    \REQUIRE AMEC $(A, \A_A)$, error $\epsilon_{\texttt{PR}} > 0$
	    \STATE Set $\rho \leftarrow 2 \psi(\phi_{\texttt{FindAMEC}}(\beta))$ \hfill$\COMMENT{\rho \sim \|P - \tilde{P}\|^{-1}_1}$
	    \REPEAT
	    \STATE Set $\rho \leftarrow \frac{\rho}{2}$
	    %
	    \STATE Sample $\psi^{-1}(\rho)$ times $\forall (s,a) \in A \times \A_A$\\
	    %
	    \STATE $v', v, \tilde{P} \leftarrow \texttt{VI}(\Bellman^{\alpha}_{\texttt{PR}}, d_{\texttt{PR}}, \epsilon^{\Bellman}_{\texttt{PR}})$ \hfill$\COMMENT{v' = \Bellman_{\texttt{PR}}^{\alpha}v}$
	    %
	\UNTIL{$\rho > \frac{\epsilon_{\texttt{PR}} (1 - \Delta(\tilde{P}))}{3 |A| c_{\max}}$ \hfill$\COMMENT{\|P - \tilde{P}\|_1 \textit{ small} }$}
	\smallskip
	\STATE Set policy $\pi \leftarrow \eta$-greedy policy w.r.t.~$v'$
	\STATE Set gain $g_\pi \leftarrow \frac{1}{2}\left(\max(v' - v) + \min(v'-v)\right)$ 
	\vspace{-0.13in}
	\RETURN $\pi, g_\pi$
	\end{algorithmic}
	\end{small}
\end{algorithm}
\end{minipage}
\vspace{-.5cm}
\end{wrapfigure}
policy), we use Alg.~\ref{algo:plan_recurrent} with (extended) relative value iteration ($\texttt{VI}$, Alg.~\ref{algo:VI} in appendix) using the optimistic Bellman operator $\Bellman^\alpha_{\texttt{PR}}$ (see Table \ref{tab:param}, we discuss $\alpha$ in next paragraph).
Let $\pi_v$ denote the greedy policy w.r.t.~the fixed point  $v =\Bellman^\alpha_{\texttt{PR}} v$ ($v$ is the optimistic value estimate). 
Using the $\eta$-greedy policy, $\pi \equiv (1-\eta)\pi_v + \eta\texttt{Unif}(\A_{A})$ (Alg.~\ref{algo:plan_recurrent}, Line 7), together with  $P_{\pi}$, makes $A$ recurrent: $s^\ast \in A$ is visited infinitely often and $\Prob{\pi \models \varphi| s_0 \in A} = 1$. Since $\eta$ can be arbitrarily small (Lemma \ref{lem:eta_greedy_approx}), then $g_{\pi} \approx g_{\pi_v}$ and $\pi$ is both cost and probability optimal. As intuited in Section \ref{subsec:intuition}, $\pi$ has full support over $\A_A$ but is nearly deterministic.\footnote{Typically, RL settings admit a fully deterministic optimal policy, but for LTL constrained policy optimization the optimal policy may not be deterministic (although can be very nearly so). See Cost Optimality in Section \ref{subsec:intuition}}

$\texttt{VI}$ in Line 5 of Alg.~\ref{algo:plan_recurrent} is an iterative procedure (Alg.~\ref{algo:VI} in appendix), and terminates via $d_{\texttt{PR}} < \epsilon^\Bellman_{\texttt{PR}}$ (Table \ref{tab:param}).
Convergence of extended $\texttt{VI}$ is guaranteed \cite{puterman2014markov, Jaksch2010, FruitUCRL22020}, so long as the dynamics, $\tilde{P} = \arg\min_{p \in \mathcal{P}(s,a)} p^T v$, achieving the inner minimization of $\Bellman^\alpha_{\texttt{PR}}$ are aperiodic -- hence the aperiodicity transform $\alpha\in(0,1)$ in $\Bellman^\alpha_{\texttt{PR}}$ \cite{puterman2014markov}. Computing $\tilde{P}$ can be done efficiently \cite{Jaksch2010} (Alg.~\ref{algo:inner_min_P} in appendix). For stability, we shift each entry of 
$v_n$ by the value of the first entry 
$v_n(0)$ \cite{bertsekas2011dynamic}. 

Alg.~\ref{algo:plan_recurrent} returns the average gain cost $g_\pi$ of policy $\pi$ when we have enough samples for each state-action pair in $(A, \A_A)$ to verify that $n > \psi^{-1}\left( \frac{\epsilon_{\texttt{PR}} (1 - \Delta(\tilde{P}_\pi))}{3 |A| c_{\max}} \right)$ where $\Delta(\tilde{P}_\pi) = \frac{1}{2} \max_{ij} \sum_{k} |\tilde{p}_{ik} - \tilde{p}_{jk}|.$
Here, $\Delta(\tilde{P}_\pi)$ is an easily computable measure on the ergodicity of the Markov chain $\tilde{P}_\pi$ \cite{Cho_Meyer_2001}. We track $\psi(n)$ (recall Def.~\ref{def:high_prob_event}) via a variable $\rho$ and sample $\psi^{-1}(\rho) \approx \frac{1}{\rho^2}$ (see Lemma \ref{lem:samples_req}) samples from each state-action pair in $(A, \A_A)$ (Alg.~\ref{algo:plan_recurrent}, Line 4). We halve $\rho$ each iteration (Alg.~\ref{algo:plan_recurrent}, Line 3) and convergence is guaranteed because $\rho$ will never fall below some unknown constant $\frac{\epsilon_{\texttt{PR}} (1- \bar{\Delta}_A)}{6|A|c_{\max}}$ (see Lemma \ref{lem:simulation_avg}); 
the halving trick is required because $\bar{\Delta}_A$ is unknown a priori.

\begin{restatable}[\texttt{PR} Convergence \& Correctness, Informal]{prop}{gainConvCorrectness}
\label{prop:plan_recurrent} 
Let $\pi_A$ be the gain-optimal policy in AMEC $(A, \A)$. Algorithm \ref{algo:plan_recurrent} terminates after at most $\log_2\left(\frac{6|A|c_{\max}}{\epsilon_{\texttt{PR}} (1- \bar{\Delta}_A)}\right)$ repeats, and collects at most $n = \tilde{\mathcal{O}}(\frac{|A|^2c_{\max}^2}{\epsilon_{\texttt{PR}}^2 (1- \bar{\Delta}_A)^2})$ samples for each $(s,a) \in (A, \A_A)$. The $\eta$-greedy policy $\pi$ w.r.t.~$v'$ (Alg.~\ref{algo:plan_recurrent}, Line 5) is gain optimal and probability optimal:
 $   |g_{\pi} - g_{\pi_A} | < \epsilon_{\texttt{PR}},$  $\Prob{\pi \models \varphi | s_0 \in A} = 1.$
\end{restatable}

\begin{table*}[t]
\vspace{-0.15in}
\caption{Subroutine Operators and Parameters for Value Iteration}
\label{tab:param}
\vspace{-0.1in}
\begin{center}
\begin{small}
\begin{tabular}{ll}
\toprule
Op/Param & Description \\
 \midrule
$\Bellman^\alpha_{\texttt{PR}}v(s)$ & $\min_{a \in \A_A(s)} \big(\mathcal{C}(s,a) + \alpha \min_{p \in \mathcal{P}(s,a)} p^T v \big) + (1-\alpha) v(s) \quad \forall s \in A$ \\
$d_{\texttt{PR}}(v_{n+1}, v_n) < \epsilon^{\Bellman}_{\texttt{PR}}$ & $\max_{s \in A}(v_{n+1}(s) - v_n(s)) - \min_{s \in A}(v_{n+1}(s) - v_n(s)) < \frac{2\epsilon_{\texttt{PR}}}{3}$ \\
$\Bellman_{\texttt{PT}} v(s)$ & $\begin{cases} \min\left\{\min_{a \in \A_A(s)} \left( \mathcal{C}(s,a) + \min_{p \in \mathcal{P}(s,a)}p^T v \right), \bar{V} / \epsilon_{\varphi}\right\},& \quad s\in \SState\setminus \cup_{i=1}^k A_i \\
\lambda g_i,& \quad s \in A_i
\end{cases}$\\
$d_{\texttt{PT}}(v_{n+1}, v_n) < \epsilon^{\Bellman}_{\texttt{PT}}$ & $\|v_{n+1} - v_n\|_1 < c_{\min} \epsilon_{\texttt{PT}} \epsilon_{\varphi} / (4 \bar{V})$ \\
 \bottomrule
\end{tabular}
\end{small}
\end{center}
\vspace{-0.1in}
 \end{table*}

\begin{wrapfigure}{l}{0.49\textwidth}
\vspace{-0.8cm}
\begin{minipage}{.48\textwidth}
\begin{algorithm}[H]
	\begin{small}
	\caption{$\texttt{PlanTransient}$ ($\texttt{PT}$) } 
	\label{algo:plan_transient}
	\begin{algorithmic}[1]
	    \REQUIRE States \& gains: $\{(A_i, g_i)\}_{i=1}^k$, err. $\epsilon_{\texttt{PT}} > 0$
	    \STATE Set $V_T(s) = \lambda g_i$ for $s \in A_i$ \hfill$\COMMENT{\textit{Terminal costs}}$
	    %
	    \STATE Sample $\phi_{\texttt{PT}}$ times $\forall (s,a)\in (\SState \setminus \cup A_i) \times \A$ 
	 %
	    \STATE $v', v, \tilde{P}  \leftarrow \texttt{VI}(\Bellman_{\texttt{PT}}, d_{\texttt{PT}}, \epsilon^\Bellman_{\texttt{PT}}, V_T)$ \hfill$\COMMENT{v' = \Bellman_{\texttt{PT}}v}$
	    %
	\STATE Set $\pi \leftarrow $greedy policy w.r.t $v'$
	\RETURN $\pi$
	\end{algorithmic}
	\end{small}
\end{algorithm}
\end{minipage}
\vspace{-.2cm}
\end{wrapfigure}
\textbf{PlanTransient (\texttt{PT}).} 
This is the stochastic shortest path (SSP) reduction step that finds a policy from the initial state $s_0$ to the AMECs (Alg.~\ref{algo:plan_transient}).
The main algorithmic tool used by  \texttt{PlanTransient}  is similar to that of \texttt{PlanRecurrent}: it also uses extended value iteration ($\texttt{VI}$, Alg.~\ref{algo:VI} in appendix) but with a different optimistic Bellman operator $\Bellman_{\texttt{PT}}$ (Table \ref{tab:param}), and then returns a (fully deterministic) greedy policy w.r.t.~the resulting optimistic value $v$ (Alg.~\ref{algo:plan_transient}, Line 4).
$\Bellman_{\texttt{PT}}$ is used to calculate the highest probability, lowest cost path to the AMECs (Alg.~\ref{algo:plan_transient}, Line 3).

Since rejecting end components might exist (see $A_3$ from Figure \ref{fig:intuition} (Left)),
a trajectory may end up stuck and accumulate cost indefinitely, and so we must bound $\|v\|_\infty < \bar{V}/\epsilon_{\varphi}$ to prevent blow up. In Prop.~\ref{prop:select_V}, we show how to select $\bar{V}$ such that $\pi$ will reach the target states (in this case, the AMECs), first with high prob and then with lowest cost. The existence of such a bound on $\|v\|_\infty$ was shown to exist, without construction, in \cite{Kolobov2012DeadEnd}. In practice, choosing a large $\bar{V}$ is enough. 

The terminal costs $V_T$ (Alg.~\ref{algo:plan_transient}, Line 1) together with Bellman equation $\Bellman_{\texttt{PT}}$ has value function $\tilde{V}_{\pi} \approx p(J_{\pi} + \frac{1}{p}\sum_{i=1}^k p_i g_{\pi_i}) + (1-p) \bar{V}/\epsilon_{\varphi} \approx V_{\pi}$, relating to $V_{\pi}$ \eqref{eqn:main_problem}, see Section \ref{sec:app:overview}. Here, $p_i = \Prob{\pi \text{ reaches } A_i} \equiv \mathbb{E}_{\tau \sim \Tau_{\pi}^P}[1_{\exists s \in \tau \text{ s.t } s \in A_i}]$ and $p = \sum_{i=1}^k p_i$. $\texttt{VI}$ converges when $d_{\texttt{PT}} < \epsilon_{\texttt{PT}}$ (see Table \ref{tab:param}). Convergence of extended VI for SSP is guaranteed \cite{Tarbouriech2021, Kolobov2012DeadEnd}. The number of samples required for each state-action pair $(s,a) \in (\SState \setminus \cup A_i) \times \A$ is $\phi_{\texttt{PT}} = \psi^{-1}\left(\frac{c_{\min} \epsilon_{\texttt{PT}}
\epsilon_{\varphi}^2 
}{14 |\SState \setminus \cup_{i=1}^k A_i| \bar{V}^2}\right)$. 
\begin{restatable}[\texttt{PlanTransient} Convergence \& Correctness, Informal]{prop}{transientConvCorrectness}
\label{prop:plan_transient}
Denote the cost- and prob-optimal policy as $\pi'$. After collecting at most $n = \tilde{\mathcal{O}}(\frac{|\SState \setminus \cup_{i=1}^k A_i|^2 \bar{V}^4}{c_{\min}^2 
\epsilon_{\texttt{PT}}^2 
\epsilon_{\varphi}^4})$ samples for each $(s,a) \in (\SState \setminus \cup_{i=1}^k A_i) \times \A$, the greedy policy $\pi$ w.r.t.~$v'$ (Alg.~\ref{algo:plan_transient}, Line 3) is both cost and probability optimal:
\begin{equation*}
    \|\tilde{V}_{\pi} - \tilde{V}_{\pi'} \| < \epsilon_{\texttt{PT}}, \quad |\Prob{\pi \text{ reaches } \cup_{i=1}^k A_i} - \Prob{\pi' \text{ reaches } \cup_{i=1}^k A_i}| \leq \epsilon_{\varphi}.
\end{equation*}
\end{restatable}


\section{End-To-End Guarantees}
\label{sec:analysis}

The number of samples necessary to guarantee an $(\epsilon_{V}, \epsilon_{\varphi}, \delta)$-PAC approximation to the cost-optimal and probability-optimal policy relies factors: $\beta$ (lower bound on the min. non-zero transition probability of $P$), $\{c_{\min}, c_{\max}\}$ (bounds on the cost function $\mathcal{C}$), $\bar{\Delta}_{A_i}$ (worst-case coefficient of ergodicity for EC $(A_i, \A_{A_i})$), $\bar{V}$ (upper bound on the value function), and $\lambda$ (tradeoff factor).

\begin{restatable}[Sample Complexity]{thm}{generalization}
\label{thm:generalization}
Under the event $\mathcal{E}$, Assumption \ref{assump:lower_bound} and \ref{assump:cost_function}, after 
\begin{equation*}
n = \tilde{\mathcal{O}}\left(\frac{1}{\beta} + \frac{1}{\epsilon_V^2} \left( \frac{|\SState|^2 \bar{V}^4}{ c_{\min}^2 \epsilon_{\varphi}^4} + \lambda^2 \sum_{i=1}^k \frac{|A_i|^2 c_{\max}^2}{(1-\bar{\Delta}_{A_i})^2} \right) \right)
\end{equation*}
samples\footnote{
The lower bound relating to $\beta$ from \cite{Littman2017} is $\Omega(\frac{\log(2\delta)}{\log(1-\beta)})$ whereas ours is $\tilde{O}(\frac{1}{\beta})$. We conjecture that $\tilde{\Omega}(\frac{1}{\beta})$ samples is required. See Appendix Section \ref{app:conjecture}.} are collected from each state-action pair, the policy $\pi$ returned by Algorithm \ref{algo:main} is, with probability $1-\delta$, simultaneously $\epsilon_V$-cost optimal and $\epsilon_\varphi$-probability optimal, satisfying:
\begin{equation}\label{eq:guarantees}
(i) \;\; |\Prob{\pi \models \varphi} - \Prob{\pi^\ast \models \varphi}| \leq \epsilon_{\varphi} \quad (ii)\;\; \|V_\pi - V_{\pi^\ast}\|_{\infty} < \epsilon_{V}.
\end{equation}
\end{restatable}

With a sufficiently large $\lambda$ (which may not be verifiable in practice), $\pi$ is also gain optimal. 



\begin{restatable}[Gain (Average Cost) Optimality]{cor}{gain}
\label{cor:lambda}
There exists  $\lambda^\ast > 0$ s.t.~for $\lambda > \lambda^\ast$, the policy $\pi$ returned by Alg.~\ref{algo:main} satisfies \eqref{eq:guarantees}, $g_{\pi} = \arg\min_{\pi' \in \Pi_{\max}} g_{\pi'}$, and is probability and gain optimal.
\end{restatable}

The high-level structure of our analysis follows the algorithm structure in Section \ref{sec:algo}, via composing the constituent guarantees. To complete the analysis, we develop some technical tools which may be of independent interest, including a gain simulation Lemma \ref{lem:simulation_avg} and an $\eta$-greedy optimality Lemma \ref{lem:eta_greedy_approx}.
For ease of exposition, we also ignore paths between AMECs (see Appendix \ref{app:sec:blocking}).
\vspace{-.3cm}
\section{Empirical Analysis} \label{sec:experiments}
\vspace{-.3cm}

We perform experiments in two domains: (1) Pacman domain where an agent finds food and indefinitely avoids a ghost; (2) discretized version of mountain car (MC) \cite{openAI} where the agent must reach the flag. Our goal is to understand whether: (i)  our LCP approach (Alg.\ref{algo:main}) produces 
competitive polices; (ii) LCP can work in continuous state spaces through discretization; (iii) LCP can enjoy efficient sample complexity in practice. For a baseline, we use Logically Constrained RL (LCRL, \cite{Hasanbeig2018lcrl}), which is a Q-learning approach to LTL-constrained PO in unknown MDPs. We also do heavy cost shaping to LCRL as another baseline. See App \ref{app:experiments} for more details, experiments, and figures.

\subsection{Results}

\textbf{Competitiveness of the policy in full LTL specs?} The probability of LCP satisfying the LTL \begin{wrapfigure}{r}{0.55\textwidth}
\vspace{-0.15in}
\includegraphics[width=\linewidth]{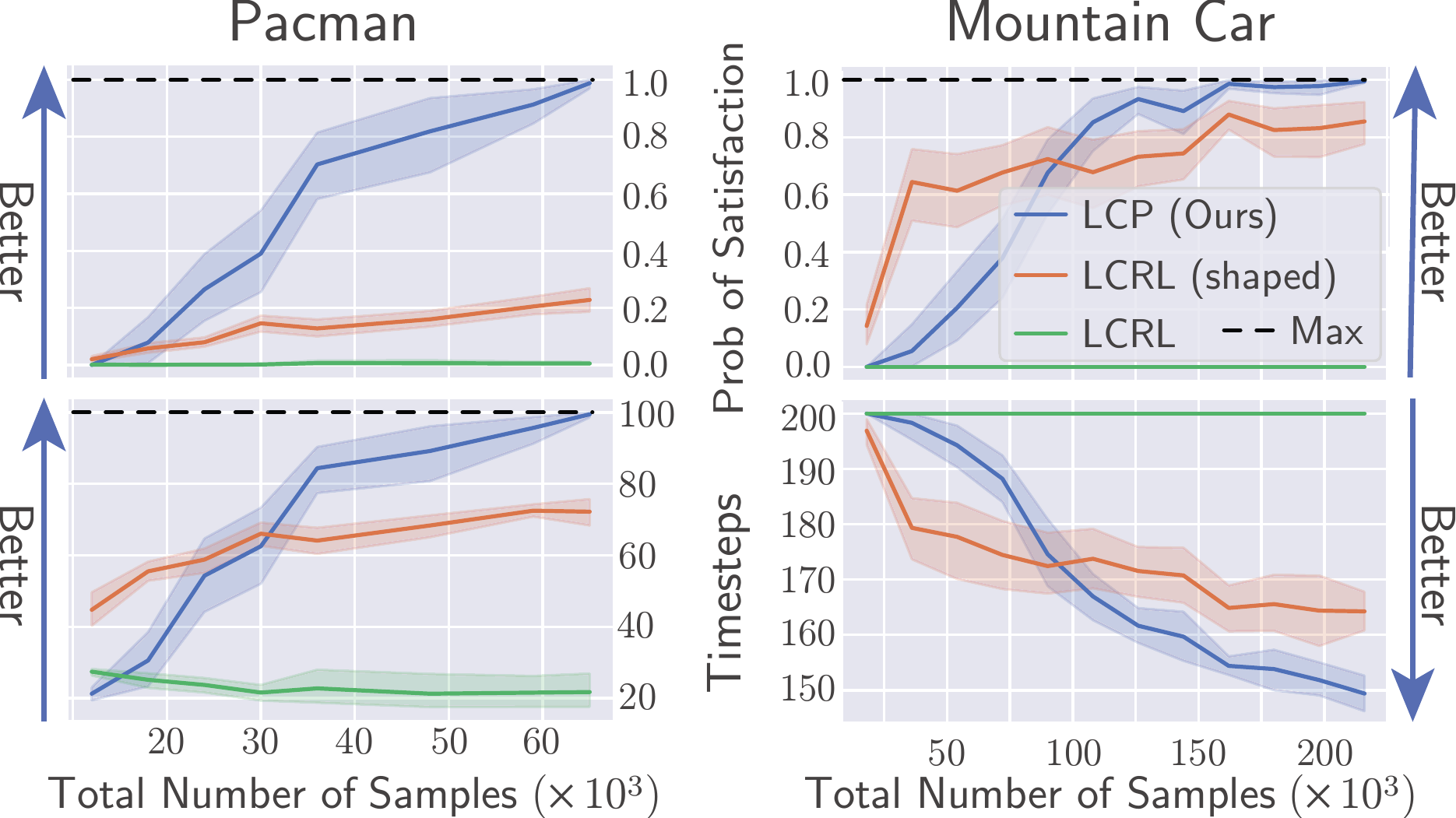}
\caption{\textit{Results. (Left Column) Pacman. $\varphi$ is to eventually collect food and always avoid the ghost. We let the system run for a maximum of 100 timesteps. (Right Column) Discretized Mountain Car (MC). $\varphi$ is to eventually reach the flag.
}}
 \label{fig:experiments}
 \vspace{-.2in}
\end{wrapfigure}
spec in Figure \ref{fig:experiments} (Left) approaches $1$ much faster than the two baselines. The returned policy collects the food quickly and then stays close, but avoids, the ghost. Any policy that avoids the ghost is equally good, as we have not incentivized it to stay far away. LCRL redefines cost as 1 if the LTL is solved and 0 otherwise, which is too sparse and learning suffers. Indeed, shaped LCRL performs better than straight LCRL.

\textbf{Performance in continuous state space?} Similarly, the probability of satisfying the LTL spec in Figure \ref{fig:experiments} (Right) goes up to $1$. However, here the LCRL (shaped) baseline performs relatively well as it is being given ``breadcrumbs'' for how to solve the task.
Our algorithm performs well without needing any cost shaping. Standard LCRL fails to learn. This experiment demonstrates that our method can be used even in discretized continuous settings.

\textbf{Sample Complexity?} Our theory is quite conservative w.r.t.~empirical performance. In Pacman (Figure \ref{fig:experiments}, Left), Thm.~\ref{thm:generalization} suggests $\approx350$ samples per $(s,a)$ pair just to calculate the AMECs. Empirically, LCP finds a good policy after $11$ samples per $(s,a)$ pair ($\sim 66k/6k$ samples/pair). 

\textbf{Other Considerations.} One of the strengths and potential drawbacks of LTL is its specificity. If a $\varphi$, for a truly infinite horizon problem, is to ``eventually'' do something, then accomplishing the task quickly is not required. 
As a finite horizon problem, in MC (Fig.~\ref{fig:experiments}, Right) SSP finds the fastest path to the goal. In contrast, since any stochastic policy with full support will ``eventually'' work, the policy returned by LCP for Fig \ref{fig:motivation} (Left) (Fig. \ref{fig:intuition} Center, \& App Fig. \ref{fig:experiments_policy_types}) may take exponential time to complete a single loop. Two straightforward ways to address this issue are: (a) including explicit time constraints in $\varphi$; and (b) cost shaping to prefer policies reaching some $s^\ast$ quickly and repeatedly. Unlike standard cost-shaping, $\varphi$ satisfaction is still guaranteed since the cost is decoupled from $\varphi$.

\section{Related Work}\label{sec:related_work}
\vspace{-.2cm}

\textbf{Constrained Policy Optimization.}
One attempt at simplifying cost functions is to split the desired behaviors from the required behaviors. The desired behaviors remain as part of the cost function while the required behaviors are treated as constraints. 
Recent interest in constrained policy optimization within the RL community has been related to the constrained Markov Decision Process (CMDP) framework \cite{Altman99, fqe, Achaim_2017,miryoosefi2019reinforcement}. 
This framework enables clean methods and guarantees, but enforces expected constraint violations rather than absolute constraint violations. Setting and interpreting constraint thresholds can be very challenging, and inappropriate in safety-critical problems \cite{Le2018}. 

\textbf{LTL + RL.} Recently, LTL-constrained policy optimization has been developed as an alternative to CMDPs \cite{Littman2017}. Unlike CMPDs, the entire task is encoded into an LTL expression and is treated as the constraint. Q-learning variants when dynamics are unknown and Linear Programming methods when dynamics are known are common solution concepts \cite{Sadigh2014, Hasanbeig2018lcrl, Bozkurt2020Q, Cai2021, Ding2014}. The Q-learning approaches rely on proper, unknowable tuning of discount factor for their guarantees. Theoretically oriented works include \cite{FuLTLPAC, Wolff2012RobustControl}. While providing PAC-style guarantees, the assumptions made in these works rely on unknowable policy-environment interaction properties. We make no such assumptions here.

Another solution technique is employing reward machines \cite{Icarte2020RewardMachine, Camacho2019LTLAndBeyond, Vaezipoor2021LTL2Action} or high-level specifications that can be translated into reward machines \cite{jothimurugan2019composable}. These works are generally empirical and handle finite or repeated finite problems (episodic problems at test time); they can only handle a smaller set of LTL expressions, specifically regular expressions. Our work handles $\omega$-regular expressions, subsuming regular expressions and requires a nontrivial leap, algorithmically and theoretically, to access the broader set of allowable expressions. Many problems are $\omega$-regular problems, but not regular, such as liveness (something good will happen eventually) and safety (nothing bad will happen forever). 
The works that attempt to handle full LTL expressibility redefine reward as 1 if the LTL is solved and 0 otherwise; the cost function of the MDP is entirely ignored.



\textbf{Verification and Planning.}  As an alternative to our approach, one might consider LTL satisfaction verification  and extend it to an optimization technique by checking every policy (which will naively take an exponential amount of samples to verify a single policy \cite{Bradzil2014, Kretinsky2019}). Many verification approaches exist \cite{prism2011, modelchecking, Agha2018ModelChecking, Younes2011, Lassaigne2006, Herault2004} and among the ones that do not assume known dynamics, the verification guarantees rely on quantities as difficult to calculate as the original verification problem itself \cite{Kretinsky2019}.

\section{Discussion}

We have presented a novel algorithm, LCP, for policy optimization under LTL constraints in an unknown environment. We formally guarantee that the policy returned by LCP simultaneously has minimal cost with respect to the MDP cost function and maximal probability of LTL satisfaction. Our experiments verify that our policies are competitive and our sample estimates conservative.

The assumptions we make are strong, but to the best of our knowledge, are the most relaxed amongst tractable model-based algorithms proposed for this space. Model-free algorithms (Q-learning) have less stringent assumptions but do not come with the kind of guarantees that our work has and largely ignore the cost function, solving only part of the problem. An interesting future direction would be to extend  our work to continuous state and action spaces and settings with function approximation.

\begin{small}
\textbf{Acknowledgements.} Cameron Voloshin is funded partly by an NSF Graduate Fellowship and a Kortschak Fellowship. This work is also supported in part by NSF \#1918865, ONR \#N00014-20-1-2115, and NSF \#2033851.
\end{small}

\begin{small}
\bibliography{neurips_2022}
\bibliographystyle{plain}
\end{small}

\newpage
\tableofcontents
\appendix
\newpage
\section{Notation and Overview}
\label{app:glossary}

\begin{figure}[!h]
\begin{minipage}[t]{1\textwidth}
\vspace{-.1in}
\captionof{table}{Glossary of terms}
\label{tab:glossary}
\vspace{-0.05in}
\begin{center}
\begin{small}
\centerline{
\begin{tabular}{ll}
\toprule
 Acronym & Term \\
 \midrule
RL & Reinforcement Learning \\
PO & Policy Optimization \\
LTL & Linear Temporal Logic \\
MDP & Markov Decision Process \\
LDBA & Limit Determinisitic Buchi Automaton \\
AMEC/MEC/EC & Accepting MEC, Maximal EC, End Component \\
\hline
LCP & LTL Constrained Planning, Algo \ref{algo:main} \\
$\texttt{FindAMEC}$ & Subroutine to assist in finding AMECs \\
$\texttt{PlanRecurrent}, \texttt{PR}$ & Subroutine to plan in AMECs, Algo  \ref{algo:plan_recurrent} \\
$\texttt{PlanTransient}, \texttt{PT}$ & Subroutine to plan to AMECs, Algo  \ref{algo:plan_transient} \\
$\texttt{NoBlockPlanTransient}, \texttt{NB-PT}$ & Subroutine to plan to AMECs, Algo  \ref{algo:no_block_plan_transient} \\
$\texttt{VI}$ & Value Iteration Subroutine, Algo  \ref{algo:VI} \\
\hline
AP & Atomic Proposition \\
$\Sigma$ & Alphabet $\Sigma = 2^{AP}$ \\
$\SState$ & State Space \\
$\mathcal{A}$ & Action Space. $\A(s)$ allowable actions in state $s$. \\
$\mathcal{A}_{A}$ & Restricted Action Space. $\A_A(s) \subseteq \A(s)$ allowable actions in state $s \in A \subseteq \SState$. \\
$P$ & Transition Function \\
$P_{\pi}$ & Markov Chain induced by $\pi$ in $P$ \\
$\mathcal{C}$ & Cost Function \\
$\mathcal{X}$ & Product-MDP \\
$\tau$ & Run or trajectory in an MDP \\
$\pi$ & Policy \\
$\eta$-greedy w.r.t $\pi$ in $A \subseteq \SState$ & $= (1-\eta)\pi + \eta\texttt{Unif}(\A_A)$ with $0 \leq \eta \leq 1$ \\
\hline
$\varphi$ & LTL Specification/Formula/Task \\
$\Prob{\pi \models \varphi}$ & Probability that a policy satisfies the task \\
$\Pi, \Pi_{\max}$ & Class of stochastic policies, $\pi \in \Pi$ with maximal $\Prob{\pi \models \varphi}$ \\
$\beta$ & Lower bound on minimum, nonzero transition probability \\
$D$ & Dataset tracking all tuples $(s,a,s')$ simulated \\
$\widehat{P}$ & Empirical estimate of $P$ from data in $D$ \\
$\tilde{P}$ & Optimistic dynamics returned by \texttt{VI} \\
$\mathcal{P}$ & Plausible transition functions consistent with all the information gathered in $D$ \\
$\mathcal{E}$ & High probability event \\
$n(s,a)$ & Number of samples accumulated in $(s,a)$. Also denoted $n$ \\
$\psi(n)$ & Error bound on $\max_{s' \in \SState} |\widehat{P}(s,a,s') - P(s,a,s')|$ \\
$\psi^{-1}(\rho)$ & Number of samples $n(s,a)$ necessary to achieve $\max_{s' \in \SState}  |\widehat{P}(s,a,s') - P(s,a,s')| < \rho$ \\
\hline
$\epsilon_V$ & Cost-optimality tolerance wrt. main problem \eqref{prob:main} \\
$\epsilon_{\varphi}$ & Prob-optimality tolerance wrt. main problem \eqref{prob:main} \\
$\epsilon_{\texttt{PR}}$ & error input into $\texttt{PR}$, $\epsilon_{\texttt{PR}} = \frac{\epsilon_V}{7 \lambda}$ \\
$\epsilon_{\texttt{PT}}$ & error input into $\texttt{PT}$, $\epsilon_{\texttt{PT}} = \frac{2\epsilon_V}{9}$ \\
$\epsilon^{\Bellman}_{\texttt{PR}}$ & Convergence condition for VI in $\texttt{PR}$, $\epsilon^{\Bellman}_{\texttt{PR}}= \frac{2\epsilon_{\texttt{PR}}}{3}$\\
$\epsilon^{\Bellman}_{\texttt{PT}}$ & Convergence condition for VI in $\texttt{PT}$, $ \epsilon^{\Bellman}_{\texttt{PT}} = \frac{c_{\min} \epsilon_{\texttt{PT}}\epsilon_{\varphi}}{4 \bar{V}}$ \\
\hline
$\alpha$ & Aperiodicity Coefficient $\alpha \in (0,1)$. $P_{\alpha, \pi} = \alpha P_{\pi} + (1-\alpha) I$\\
$\Bellman^\alpha_{\texttt{PR}}$ &  $\Bellman^\alpha_{\texttt{PR}}v(s) = \min_{a \in \A_A(s)} \big(\mathcal{C}(s,a) + \alpha \min_{p \in \mathcal{P}(s,a)} p^T v \big) + (1-\alpha) v(s) \quad \forall s \in A$ \\
$\Bellman_{\texttt{PT}}$ &  $\Bellman_{\texttt{PT}} v(s) = \begin{cases} \min\left\{\min_{a \in \A_A(s)} \left( \mathcal{C}(s,a) + \min_{p \in \mathcal{P}(s,a)}p^T v \right), \bar{V}\right\},& \quad s\in \SState\setminus \cup_{i=1}^k A_i \\
\lambda g_i,& \quad s \in A_i
\end{cases}$\\
$d_{\texttt{PR}} $ & Convergence operator for $\texttt{PR}$, $d_{\texttt{PR}}(v', v) = \max_{s \in A}(v'(s) - v(s)) - \min_{s \in A}(v'(s) - v(s))$ \\
$d_{\texttt{PT}}$ & Convergence operator for $\texttt{PT}$, $d_{\texttt{PT}}(v_{n+1}, v_n) = \|v_{n+1} - v_n\|_1$ \\
$V_T$ & Terminal costs. $V_T = 0$ by default \\
\hline
$\lambda$ & Tradeoff between $g_{\pi}$ and $J_{\pi}$ \\
$J_\pi$ & Transient cost, conditioned on runs satisfying $\varphi$ \\
$g_\pi$ & Gain, Average-cost, conditioned on runs satisfying $\varphi$ \\
$\bar{V}$ & Upper bound on $J_{\pi}$ for any $\pi \in \Pi$ \\
\hline
$\Delta(M)$ & Coefficient of Ergodicity of matrix
$M$, $\Delta(M) = \frac{1}{2} \max_{ij} \sum_{k} |M_{ik} - M_{jk}|$ \\ 
$\bar{\Delta}_{A_i}$ & Worst-case coefficient of ergodicty in $A_i$ \\ 
\end{tabular}
}
\end{small}
\end{center}

  \end{minipage}
  \end{figure}
  
 \clearpage
 
 \subsection{Overview}\label{sec:app:overview}
 
There is a lot of notation that we will be using to get through the analysis. It is important to distinguish the following:
 
\begin{figure}[!h]
\begin{minipage}[t]{1\textwidth}
\vspace{-.1in}
\captionof{table}{Policies and Probabilities}
\label{tab:policies_and_probs}
\vspace{-0.05in}
\begin{center}
\begin{small}
\centerline{
\begin{tabular}{ll}
\toprule
 Acronym & Term \\
 \midrule
$\pi^\ast$ & Optimal policy w.r.t \eqref{eqn:main_problem} \\
$\pi$ & Policy returned by LCP (Algo \ref{algo:main}) \\
$\pi_{A_i}$ & Gain and Prob-optimal policy in AMEC $(A_i, \A_{A_i})$ in dynamics $P$ \\
$\tilde{\pi}_{A_i}$ & Gain and Prob-optimal policy in AMEC $(A_i, \A_{A_i})$ in dynamics $\tilde{P}$ \\
$\pi_{i}$ & A policy in states $A_i$ of an AMEC $(A_i, \A_{A_i})$, ignoring what $\pi$ does outside of $A_i$ \\
\hline
$p^{\pi}$, $p^\ast$ & $\Prob{\pi \models \varphi}$ and $\Prob{\pi^\ast \models \varphi}$. Also denoted $p^{\pi}$ and $p^{\pi^\ast}$ \\
$p^{\pi}_i$, $p^{\pi^\ast}_i$ & $\Prob{\pi \text{ reaches } A_i}, \Prob{\pi^\ast \text{ reaches } A_i} \geq 0$ denoted $p_i$, $p^\ast_i$ (resp), $\sum_{i=1}^k p_i = p$, $\sum_{i=1}^k p^\ast_i = p^\ast$
\end{tabular}
}
\end{small}
\end{center}
\end{minipage}
\end{figure}

\begin{figure}[!h]
\begin{minipage}[t]{1\textwidth}
\vspace{-.1in}
\captionof{table}{Gains}
\label{tab:gains}
\vspace{-0.05in}
\begin{center}
\begin{small}
\begin{tabular}{ll}
\toprule
 Term & Description. (Subscript $i$ or $A_i$ denotes ``in AMEC $(A_i, \A_{A_i})$'') \\
 \midrule
$\widehat{g}^{\tilde{P}}_{\tilde{\pi}_{A_i}}$ & Approximated gain of (greedy) optimal policy $\tilde{\pi}_{A_i}$ under optimistic dynamics $\tilde{P}$ \\
$g^{P}_{\pi_i}$ & Actual gain of policy $\pi_i$ ($\eta$-greedy version of policy from \texttt{PR}) under true dynamics $P$ \\
$g_{\pi_{A_i}}^{P}$ & Gain of optimal policy $\pi_{A_i}$ under dynamics $P$ \\
\end{tabular}
\end{small}
\end{center}
\end{minipage}
\end{figure}

The relationship between these gains is subtle. \texttt{PlanRecurrent} (Algo \ref{algo:plan_recurrent}) returns $\widehat{g}^{\tilde{P}}_{\tilde{\pi}_{A_i}}$ as the estimate for how good the best greedy policy will be in AMEC $(A_i, \A_{A_i})$ under dynamics $\tilde{P}$. But, we don't use the greedy policy, we use the $\eta$-greedy policy $\pi_i$. With $\pi_{A_i}$ being the true gain-optimal policy in dynamics $P$ (in AMEC $(A_i, \A_{A_i})$), then we will find the following relations: 
\begin{equation*}
    \underbrace{\widehat{g}^{\tilde{P}}_{\tilde{\pi}_{A_i}}}_{\text{Output from } \texttt{PR} \;\;\;\;} 
    \underbrace{\approx}_{\frac{\epsilon_{\texttt{PR}}^\Bellman}{2} } 
    g^{\tilde{P}}_{\tilde{\pi}_{A_i}} 
    \underbrace{\approx}_{\text{Lem } \ref{lem:eta_greedy_approx}}
     g^{\tilde{P}}_{\pi_i} 
    \underbrace{\approx}_{\text{Lem } \ref{lem:simulation_avg}}
     g^{P}_{\pi_i}
    \underbrace{\approx}_{\text{Prop } \ref{prop:plan_recurrent_formal}}
     \underbrace{g^{P}_{\pi_{A_i}}}_{\text{Actual}}
\end{equation*}

In general gains $g$ are functions of state: $g(s)$. However, it is well known \cite{puterman2014markov,FruitUCRL22020} that in communicating MDPs (each state is reachable from one another by some policy) that the gain of the optimal policy (even if determinstic) is constant -- independent of state. Since AMECs are communicating MDPs, then $\pi_{A_i}, \tilde{\pi}_{A_i}$ induce constant gains in  $P,\tilde{P}$ respectively. Lastly, the stochastic policy $\pi_i$ makes both $\tilde{P}$ and $P$ recurrent, and so the gain is also constant. We will therefore only be considering the absolute difference between gains rather than $L_\infty$ norms (as they coincide).

\begin{figure}[!h]
\begin{minipage}[t]{1\textwidth}
\vspace{-.1in}
\captionof{table}{Value Functions}
\label{tab:value_funcs}
\vspace{-0.05in}
\begin{center}
\begin{small}
\begin{tabular}{ll}
\toprule
 Acronym & Term \\
 \midrule
$V_{\pi}$ & Main objective, value function $V_{\pi, \lambda}^P=J_\pi + \lambda g_{\pi}$\\
$v$ & Approximated value of policy $\pi$ from \texttt{PT} (Algo \ref{algo:plan_recurrent}) in dynamics $\tilde{P}$ \\
$\tilde{V}^{\tilde{P}}_\pi$ & Actual value of policy $\pi$ from \texttt{PT} in dynamics $\tilde{P}$ \\
$\tilde{V}^{P}_\pi$ & Actual value of policy $\pi$ from \texttt{PT} in dynamics $P$, \\ & \hfill also denoted $\tilde{V}_\pi = p(J_{\pi} + \sum_{i=1}^k \frac{p_i}{p} \widehat{g}^{\tilde{P}}_{\tilde{\pi}_{A_i}}) + (1-p) \frac{\bar{V}}{\epsilon_\varphi}$ \\
\end{tabular}
\end{small}
\end{center}
\end{minipage}
\end{figure}

When superscripts are dropped in $V$, the dynamics are the true dynamics $P$ of the product-MDP $\X$. Once again, the relationships between these value functions is subtle. \texttt{PlanTransient} (Algo \ref{algo:plan_transient}) returns $v$ as the estimate for how good $\pi$ (the greedy policy wrt $v$) will be in reaching AMECs $\{(A_i, \A_{A_i})\}_{i=1}^k$, but is optimistic. $v$ is an approximation to $\tilde{V}^{\tilde{P}}_\pi$. Roughly speaking, we will find that they are all similar/related:  
\begin{equation*}
    \underbrace{v}_{\text{Ouput from }\texttt{PT} \;\;}
    \underbrace{\approx}_{\text{Lem } \ref{lem:EVI_bound}} \tilde{V}^{\tilde{P}}_{\pi} \underbrace{\approx}_{\text{Lem } \ref{lem:simulation_transient}}
    \tilde{V}^{P}_{\pi} 
    \underbrace{\approx}_{\text{Prop }  \ref{prop:plan_transient_formal}/\ref{prop:no_block_plan_transient}}
    \underbrace{\tilde{V}^{P}_{\pi^\ast}}_{\text{An intermediate Value Func.}}
\end{equation*}
where the last approximation has $2$ different propositions: the first allows the simplifying assumption made in the main paper regarding paths between AMECS, the second removes that assumption at the expense of increased computation. Finally,
\begin{equation}
    \|\tilde{V}^{P}_{\pi} - \tilde{V}^{P}_{\pi^\ast} \| \underbrace{\approx}_{\text{Thm } \ref{thm:generalization}} \|V^{P}_{\pi} - V^{P}_{\pi^\ast}\|
\end{equation}
which involves swapping $\widehat{g}^{\tilde{P}}_{\pi_{A_i}}$ in $\tilde{V}$ for $g^P_{\pi_{A_i}}$.
\newpage
\section{Analysis: Statements with Proof}\label{app:analysis}

\subsection{Sample Complexity Guarantee}

The number of samples necessary to guarantee an $(\epsilon_{V}, \epsilon_{\varphi}, \delta)$-PAC approximation to the cost-optimal and probability-optimal policy relies factors: $\beta$ (lower bound on the mininum non-zero transition probability of $P$), $\{c_{\min}, c_{\max}\}$ (bounds on the cost function $\mathcal{C}$), $\bar{\Delta}_{A_i}$ (worst-case coefficient of ergodicity for EC $(A_i, \A_{A_i})$), $\bar{V}$ (upper bound on the value function), and $\lambda$ (tradeoff factor).
Recall that an event $\mathcal{E}$ captures the scenario where the empirical transition function $\widehat{P}$ is close to the true transition function $P$. $\mathcal{E}$ holds with probability at least $1-\delta$, see Lem \ref{lem:event}.

\generalization*

\textbf{Comparison To RL Literature.} Before presenting the proofs, we briefly compare this guarantee with standard guarantees in model-based reinforcement learning under a generative model. It is important to note that while we show that our guarantee is a sum of $3$ terms, a tighter bound would be a max over the $3$ terms. To the best of our knowledge, the current state-of-the-art RL (with generative model) guarantee is $\mathcal{\tilde{O}}(\frac{1}{(1-\gamma)^3 \epsilon^2})$ \cite{agarwal2020model}, per 
state-action pair. Here, $H = \frac{1}{1-\gamma}$ represents the effective horizon in discounted settings. In other words, $c_{\max}H$ is the bound on (their) $\|V\|$. In our case, for the SSP reduction, the 
effective horizon is $H = \frac{\|\tilde{V}\|_\infty}{c_{\min}}$, as this
is the expected goal-reaching time in the worst-case (since we do not have any discounting). We estimate $\|\tilde{V}\|_\infty$ with upper bound $\frac{\bar{V}}{\epsilon_{\varphi}}$. Suppose we set $\epsilon = \min(\epsilon_V, \epsilon_{\varphi})$. Focusing just on the center term, we have guarantee taking the form, roughly, $\frac{|\SState|^2 H^4}{\epsilon^2}$. Here, the $|\SState|^2$ comes from a loose upper bound $\max_{s \in \SState,a \in \A}\|\widehat{P}(s,a,\cdot)\|_1 = |\SState|.$ In fact, as noted in \cite{Tarbouriech2021}, when the MDP is not too chaotic $\max_{s \in \SState,a \in \A}\|\widehat{P}(s,a,\cdot)\|_1 = \mathcal{O}(1)$. Further, by using careful variance-aware arguments from \cite{Tarbouriech2021} we can decrease the dependency from $H^4$ to $H^3$. Hence, the SSP guarantee (our center term) and the standard RL guarantee are very similar. The first term $\frac{1}{\beta}$ does not appear in standard RL literature because there is no constraint verification needed, but in practice will be dominated by the other terms. The last term is also similar to the center term. $\frac{c_{\max}}{1-\bar{\Delta}_{A_i}}$ can also be seen as an effective horizon, accumulating $c_{\max}$ cost until the accepting component sufficiently mixes. Here, $|A_i|^2 \leq |\SState|^2$ and, again, comes from the loose upper bound $\max_{s \in A_i,a \in \A_{A_i}}\|\widehat{P}(s,a,\cdot)\|_1 = |A_i|$.

\begin{proof}[\textbf{Proof} of Theorem \ref{thm:generalization}]
We begin by examining the interaction of $\pi^\ast$ with $P$. The Markov chain $P_{\pi^\ast}$ has a number, say $m$, of recurrent classes $R_1,\ldots,R_m$, sets of states that are trapping and visited infinitely often once reached. Some of the recurrent classes $R_i$ contain an accepting state $s \in \SState^\ast$, making any trajectory entering $R_i$ an accepting run, without loss of generality call these $R_1,\ldots,R_{m'}$ (we just relabel them). Let $\A_{\pi^\ast_i}(s) = \{a \in \A | \pi^\ast_i(s|a) > 0 \}$ denote the support of actions taken by $\pi^\ast_i$ 
in state $s \in R_i$. Let $\A_{\pi^\ast_i} = \{\A_{\pi^\ast_i}(s)\}_{s \in R_i}$ be the indexed action set in $R_i$. Then, by definition, $\{(R_i, \A_{\pi^\ast_i})\}_{i=1}^{m'}$ are accepting EC. By definition, 
each accepting EC $(R_i, \A_{\pi^\ast_i})$ must be contained within (or is itself) some AMEC $(A_i, \A_i)$. 

Fix some accepting EC $(R_j, \A_{\pi^\ast_j})$. We claim, without loss of generality, $(R_j, \A_{\pi^\ast_j}) = (A_i, \A_{A_i})$ for some index $i \in {1,\ldots,k}$. To show this, let $\pi_{A_i}$ be the gain optimal, and probability-optimal policy in AMEC $(A_i, \A_{A_i})$: $\pi_{A_i}$ is defined over all states $s \in A_i$ and actions $a \in \A_{A_i}$. Further, consider the modified optimal policy
\begin{equation*}
    \tilde{\pi}^\ast(s,a) = \begin{cases}
    \pi_{A_i}(s,a),& s \in A_i \\
    \pi^\ast(s,a),& \text{otherwise}.
    \end{cases}
\end{equation*}
Because $\pi_{A_i}$ is prob-optimal (ie. $\Prob{\tilde{\pi}^\ast \models \varphi | s_0 \in A_i} = 1)$ in $A_i$ then the probability $\Prob{\tilde{\pi}^\ast \models \varphi} \geq \Prob{\pi^\ast \models \varphi}$.
Further, $J_{\tilde{\pi}^\ast} \leq J_{\pi}$ because any $\tau$ that formerly passed through $R_j \setminus A_i$ now accumulates less cost. Further, $g_{\pi_{A_i}} \leq g_{\pi^\ast_i}$ by definition of optimality in AMEC $(A_i, \A_{A_i})$. Thus, $V_{\tilde{\pi}^\ast} \leq V_{\pi^\ast}$. Of course, by definition of optimality, the opposite signs hold: $V_{\tilde{\pi}^\ast} \geq V_{\pi^\ast}$ and $\Prob{\tilde{\pi}^\ast \models \varphi} \leq \Prob{\pi^\ast \models \varphi}$. Therefore  $\tilde{\pi}^\ast$ and $\pi^\ast$ are indistinguishable. 

Repeating the above argument for each $(R_j, \A_{\pi^\ast_j})$ means the accepting EC of $\pi^\ast$ are AMECS and, by definition, form some subset of all of the AMECs $\{(A_i, \A_{A_i}\}_{i=1}^k$. In other words, all accepting runs of $\pi^\ast$ reach states $\cup_{i=1}^k A_i$. Furthermore, $g_{\pi^\ast} = \sum_{i=1}^{k} \frac{p^\ast_i}{p} g^P_{\pi_{A_i}}$ where $p^\ast_{i} \geq 0$ is the probability that $\pi^\ast$ reaches $A_i$ and $\sum_{i=1}^k p^\ast_{i} = p$. 

\textbf{Property $(i)$} now follows as a direct consequence of Prop \ref{prop:plan_transient} and Prop \ref{prop:plan_recurrent}. Recall by Prop \ref{prop:plan_transient} that $|\Prob{\pi \text{ reaches } \cup_{i=1}^k A_i} - \max_{\pi' \in \Pi_{\max}} \Prob{\pi' \text{ reaches } \cup_{i=1}^k A_i}| \leq \epsilon_{\varphi}$. Prop \ref{prop:plan_recurrent} implies that once a run enters some $A_i$, the run is accepted. Remaining runs cannot be accepted since they do not reach any AMEC, the only way to be accepted. Hence $\Prob{\pi \models \varphi} = \Prob{\pi \text{ reaches } \cup_{i=1}^k A_i}$. Since we just showed that all accepting runs of $\pi^\ast$ reach some $(A_i, \A_i)$ then:
\begin{align*}
    0 \leq \Prob{\pi^\ast \models \varphi} - \Prob{\pi \models \varphi} \leq \Prob{\pi^\ast \text{ reaches } \cup_{i=1}^k A_i} - \Prob{\pi \text{ reaches } \cup_{i=1}^k A_i} \leq \epsilon_{\varphi}.
\end{align*}


To show \textbf{Property $(ii)$}, first let us define $p^{\pi}_i$ as the probability of $\pi$ reaching AMEC $(A_i, \A_{A_i})$ and, by property $(i)$, $\sum_{i=1}^k p^{\pi}_i = \sum_{i=1}^k p^\ast_{i} = p$. 
The value function given by the Bellman operator $\Bellman_{\texttt{PT}}$ (Table \ref{tab:param}) in Algorithm \ref{algo:plan_transient} takes the form
\begin{equation}\label{eq:val_func}
    \tilde{V}_{\pi}(s) = p( J_{\pi}(s) + \lambda \sum_{i=1}^k \frac{p^{\pi}_i}{p} \widehat{g}^{\tilde{P}}_{\tilde{\pi}_{A_i}} + (1-p) \frac{\bar{V}}{\epsilon_{\varphi}} 
\end{equation}
where $\widehat{g}^{\tilde{P}}_{\tilde{\pi}_{A_i}}$ are the approximated gains for end component $(A_i, \A_{A_i})$ from Algorithm \ref{algo:plan_recurrent}. To see this, there is probability $p$ that $\pi \models \varphi$ and achieves (conditional) expected cost $J_{\pi}(s) + \lambda \sum_{i=1}^k \frac{p_i}{p} \widehat{g}^{\tilde{P}}_{\tilde{\pi}_{A_i}}$ and prob $1-p$ that $\pi \not\models \varphi$ where 
all cooresponding trajectories get stuck and accumulate $\frac{\bar{V}}{\epsilon_{\varphi}}$ cost. Let $\tilde{\pi}$ now represent the optimal solution to the value function $\tilde{V}_{\pi}$ (Algo \ref{algo:plan_transient}). Therefore we claim:
\begin{align*}
    0 \leq V_{\pi} - V_{\pi^\ast} &= V_{\pi} - \tilde{V}_{\pi} + \tilde{V}_{\pi} - \tilde{V}_{\tilde{\pi}} + \tilde{V}_{\tilde{\pi}} - \tilde{V}_{\pi^\ast} + \tilde{V}_{\pi^\ast} - V_{\pi^\ast} + (1-p)\frac{\bar{V}}{\epsilon_{\varphi}}  - (1-p)\frac{\bar{V}}{\epsilon_{\varphi}}  \\
    &\leq \underbrace{|V_{\pi} - \tilde{V}_{\pi} + (1-p)\frac{\bar{V}}{\epsilon_{\varphi}}|}_{(a)} + \underbrace{|\tilde{V}_{\pi} - \tilde{V}_{\tilde{\pi}}|}_{(b)} + \underbrace{\tilde{V}_{\tilde{\pi}} - \tilde{V}_{\pi^\ast}}_{(c)} + \underbrace{|\tilde{V}_{\pi^\ast} - V_{\pi^\ast} - (1-p)\frac{\bar{V}}{\epsilon_{\varphi}}|}_{(d)} \\
    &\leq \frac{\epsilon_V}{3}  + \frac{\epsilon_V}{3} + 0 + \frac{\epsilon_V}{3} \leq \epsilon_V
\end{align*}

For $(a)$, first we note that $g_{\pi} = \sum_{i=1}^k \frac{p^{\pi}_i}{p} g^P_{\pi_i}$, by definition of conditional expectation. Let $\epsilon_{\texttt{PR}} = \frac{\epsilon_V}{7 \lambda}$. Hence,
\begin{align*}
    \underbrace{|V_{\pi} - \tilde{V}_{\pi} + (1-p)\frac{\bar{V}}{\epsilon_{\varphi}}|}_{(a)} &= | p( J_{\pi}(s) + \lambda \sum_{i=1}^k \frac{p^{\pi}_i}{p} g^P_{\pi_i}) - p( J_{\pi}(s) + \lambda \sum_{i=1}^k \frac{p^{\pi}_i}{p} \widehat{g}^{\tilde{P}}_{\tilde{\pi}_{A_i}})| \\
    &\leq \lambda \max_{i=1,\ldots,k} |g^P_{\pi_i} - \widehat{g}^{\tilde{P}}_{\tilde{\pi}_{A_i}}| \\
    &\leq \lambda \frac{4\epsilon_{\texttt{PR}}}{3}, \quad \text{Corollary } \ref{cor:plan_recurrent_approx} \\
    &\leq \frac{\epsilon_V}{3}
\end{align*}
By similar argument, for $(d)$, together with earlier argument that $g^\ast_{\pi} = \sum_{i=1}^{k} \frac{p^\ast_i}{p} g^P_{\pi_{A_i}}$ then we also have that:
\begin{align*}
    \underbrace{|\tilde{V}_{\pi^\ast} - V_{\pi^\ast} - (1-p)\frac{\bar{V}}{\epsilon_{\varphi}}|}_{(d)} &= | p( J_{\pi^\ast}(s) + \lambda \sum_{i=1}^k \frac{p^\ast_i}{p} \widehat{g}^{\tilde{P}}_{\tilde{\pi}_{A_i}}) - p( J_{\pi^\ast}(s) + \lambda \sum_{i=1}^k \frac{p^{\pi}_i}{p} g^P_{\pi_{A_i}})| \\
    &\leq \lambda \max_{i=1,\ldots,k} |g^{P}_{\pi_{A_i}} - \widehat{g}^{\tilde{P}}_{\tilde{\pi}_{A_i}}| \\
    &\leq \lambda \max_{i=1,\ldots,k} |g^{P}_{\pi_{A_i}} - g^P_{\pi_i}| + |g^P_{\pi_i} - \widehat{g}^{\tilde{P}}_{\tilde{\pi}_{A_i}}| \\
    &\leq \lambda \frac{7\epsilon_{\texttt{PR}}}{3}, \quad \text{Prop } \ref{prop:plan_recurrent_formal} \text{ and Corollary } \ref{cor:plan_recurrent_approx} \\
    &\leq \frac{\epsilon_V}{3}
\end{align*}

Further, we have $(c) \leq 0$ holds because $\tilde{\pi}$ is optimal in $\tilde{V}$ (either by assuming $\cup_{i=1}^k A_i$ is the correct choice of AMECS, or using Algo \ref{algo:no_block_plan_transient} instead of $\texttt{planTransient}$). In either case, $(b) \leq \frac{3\epsilon_{\texttt{PT}}}{2} \leq \frac{\epsilon_V}{3}$ by Prop \ref{prop:plan_transient_formal} or Prop \ref{prop:no_block_plan_transient}, where $\epsilon_{\texttt{PT}}$ is set to $\epsilon_{\texttt{PT}} = \frac{2\epsilon_V}{9}$, completing the approximation guarantee.

We now compute the number of samples, per state-action pair, required by Algorithm \ref{algo:main}. By Prop \ref{prop:support_verification}, we need $n = \tilde{\mathcal{O}}(\frac{1}{\beta})$ to verify the support of $P$. After calculating the AMECs $\{(A_i, \A_{A_i})\}_{i=1}^k$, we calculate the gain-optimal policy $\pi_i$ for each AMEC. By Prop \ref{prop:plan_recurrent}, we need $n = \tilde{\mathcal{O}}((\frac{|A_i| c_{\max}}{\epsilon_{\texttt{PR}} (1-\bar{\Delta}_{A_i})})^2) = \tilde{\mathcal{O}}((\frac{\lambda |A_i| c_{\max}}{\epsilon_V (1-\bar{\Delta}_{A_i})})^2)$ for each state-action pair in each end component $(A_i, \A_{A_i})$, since $\epsilon_{\texttt{PR}} = \frac{\epsilon_V}{7\lambda}$. Finally, for the transient policy $\pi_0$, the SSP reduction requires $n = \mathcal{\tilde{O}}((\frac{|\SState \setminus \cup_{i=1}^k A_i| \bar{V}^2}{\epsilon_{\texttt{PT}} \epsilon_{\varphi}^2 c_{\min}})^{2}) = \mathcal{\tilde{O}}((\frac{|\SState \setminus \cup_{i=1}^k A_i| \bar{V}^2}{\epsilon_V \epsilon_{\varphi}^2 c_{\min}})^{2})$ for each state-action pair outside of the AMECs, by Prop \ref{prop:plan_transient}, since $\epsilon_{\texttt{PT}} = \frac{2\epsilon_V}{9}$. A similar sample complexity is guaranteed by using Algo $\ref{algo:no_block_plan_transient}$ in place of $\texttt{PlanTransient}$, where $n = \mathcal{\tilde{O}}((\frac{|\SState | \bar{V}^2}{\epsilon_V \epsilon_{\varphi}^2 c_{\min}})^{2})$ is required in place of $n  = \mathcal{\tilde{O}}((\frac{|\SState \setminus \cup_{i=1}^k A_i| \bar{V}^2}{\epsilon_V \epsilon_{\varphi}^2 c_{\min}})^{2})$. Adding these together, yields the worst-case number of samples necessary in any state-action pair $(s,a) \in \SState \times \A$. These sample guarantees hold only when the event $\mathcal{E}$ holds, which itself holds with probability $1-\delta$ (see Lem \ref{lem:event}).

\end{proof}

\gain*
\begin{proof}[\textbf{Proof} of Corollary \ref{cor:lambda}]
Fix some $\lambda > 0$. Let $\pi' = \arg\min_{\pi \in \Pi_{\max}} g_{\pi}$. Suppose $g_{\pi'} < g_{\pi}$ but $V_{\pi, \lambda} < V_{\pi', \lambda}$, elementwise. In other words, $\pi$ is the preferred policy. Then,
\begin{align*}
    0 &\leq V_{\pi', \lambda} - V_{\pi, \lambda} \\
    &= J_{\pi'} + \lambda g_{\pi'} -  J_{\pi} - \lambda g_{\pi} \\
    &\leq \max_{\tilde{\pi} \in \Pi} J_{\tilde{\pi}} + \lambda \underbrace{(g_{\pi'} - g_{\pi})}_{< 0}
\end{align*}
since $J_{\tilde{\pi}} \geq 0$ for each $\tilde{\pi} \in \Pi$. If $\lambda > \frac{\max_{\tilde{\pi} \in \Pi} J_{\tilde{\pi}} }{g_{\pi} - g_{\pi'}}$ then we contradict $V_{\pi, \lambda} < V_{\pi', \lambda}$. In particular, if $\pi'$ is the gain optimal policy then for any $\lambda > \lambda^\ast = \frac{\max_{\tilde{\pi} \in \Pi} J_{\tilde{\pi}} }{\min_{\{\pi \in \Pi | g_{\pi} \neq g_{\pi'}\}} g_{\pi} - g_{\pi'}}$ then $\pi'$ is preferred to any other policy $\pi \in \Pi$.

\end{proof}

\newpage
\subsection{High Probability Event and Sample Requirement}
\highProbEvent*
\begin{lem}[High Probability Event holds]\label{lem:event}
The event $\mathcal{E}$ holds with probability at least $1-\delta$.
\end{lem}

\begin{proof}[\textbf{Proof} of Lemma \ref{lem:event}
]
We start with the anytime version of Theorem 4 of \cite{Maurer2009} given by Lemma 27 of \cite{TarbouriechSSPMinimax}:
\begin{equation*}
    \mathbb{P}\left[\forall \; n \geq 1, \left|\mathbb{E}[Z] - \frac{1}{n}\sum_{i=1}^n Z_i \right| > \sqrt{\frac{2 \hat{V}_n \log(4 n^2 / \delta)}{n-1} } + \frac{7 \log(4 n^2/\delta)}{3(n-1)} \right] \leq \delta,
\end{equation*}
for any $Z_i \in [0,1]$ iid. 
By re-setting $\delta \leftarrow \frac{\delta}{|\SState|^2 |\A|}$, applying union bound over all $(s,a,s') \in \SState \times \A \times \SState$, and observing that $Z_i \sim P(s,a,s')$ is a Bernoulli random variable with empirical variance $\hat{V}_n = \widehat{P}(s,a,s')(1-\widehat{P}(s,a,s'))$ yields the result:
$$
\{\forall s,a,s' \in \SState \times \A \times \SState, \forall n > 1 : \quad |P(s,a,s') - \widehat{P}(s,a,s')| \leq \psi_{sas'}(n)\} \quad \text{ holds with prob } 1-\delta
$$
Observing that $\psi_{sas'}(n) \leq \psi(n)$ for all $n > 1$ because $\psi_{sas'}(n)$ takes on a maximum when $\widehat{P}(s,a,s') = \frac{1}{2}$, completes the proof.
\end{proof}

\begin{lem}[Inverting $\mathcal{E}$]\label{lem:samples_req} Fix $(s,a,s') \in \SState \times \A \times \SState$. Under the event $\mathcal{E}$, the number of samples $\psi^{-1}(\rho)$ required to achieve $|P(s,a,s') - \widehat{P}(s,a,s')| \leq \psi_{sas'}(n) \leq \psi(n) < \rho$ is given by:
\begin{equation*}
    \psi^{-1}(\rho) = \lceil\frac{2}{\zeta^2} \log(\frac{16 |\SState|^2 |\A|}{ \zeta^4 \delta}) \rceil + 3 = \tilde{\mathcal{O}}(\frac{1}{\rho^2}),
\end{equation*}
where $\zeta \equiv \frac{-\frac{3}{7\sqrt{2}} + \sqrt{(\frac{3}{7\sqrt{2}})^2 + \frac{12}{7} \rho }}{2}.$
\end{lem}
\begin{proof}[Proof of \ref{lem:samples_req}]
We have $\psi_{sas'} < \psi(n) = \frac{x}{\sqrt{2}} + \frac{7}{3} x^2 \leq \rho$ where $x^2 = \xi(n) = \frac{\log(4 n^2 |\SState|^2 |\A| \delta^{-1})}{{n-1}}$. Solving the quadratic inequality,
we have 
\begin{equation*}
    x \leq \frac{-\frac{3}{7\sqrt{2}} + \sqrt{(\frac{3}{7\sqrt{2}})^2 + \frac{12}{7} \rho }}{2}  \equiv \zeta
\end{equation*}
Hence, we have
\begin{align*}
\frac{\log(4 n^2 |\SState|^2 |\A| \delta^{-1})}{{n-1}} &\leq \zeta^2 \\
\implies n &\geq \frac{\log(4 n^2 |\SState|^2 |\A| \delta^{-1})}{\zeta^2} + 1 \\
&= \frac{1}{\zeta^2} \log(e^{\zeta^2} 4 n^2 |\SState|^2 |\A| \delta^{-1}) \\
&= \underbrace{\frac{2}{\zeta^2}}_{c_1} \log(\underbrace{e^{\frac{\zeta^2}{2}} \sqrt{4 |\SState|^2 |\A| \delta^{-1}}}_{c_2} n) \quad (\star)
\end{align*}
By Lemma \ref{lem:invert_N}, if $n > 2 c_1 \log(c_1 c_2)$ then $n > (\star)$. Simplifying,
\begin{equation*}
    n \geq \frac{2}{\zeta^2} \log(\frac{16 |\SState|^2 |\A|}{ \zeta^4 \delta}) + 2
\end{equation*}
Selecting $\psi^{-1}(\rho) = \lceil\frac{2}{\zeta^2} \log(\frac{16 |\SState|^2 |\A|}{ \zeta^4 \delta}) \rceil + 3$ and noting that $\zeta = \mathcal{\tilde{O}}(\rho)$ completes the proof: $n = \mathcal{\tilde{O}}(1/\rho^2)$.
\end{proof}

\begin{lem}\label{lem:invert_N}(Lemma 10 of \cite{Kazerouni2017}) If $\log(c_1 c_2) \geq 1$ and $c_1,c_2 > 0$ then
\begin{equation*}
    N > 2 c_1 \log(c_1 c_2) \implies N > c_1 \log(c_2 N)
\end{equation*}
\end{lem}

\newpage
\subsection{\texttt{FindAMEC} proofs}
\begin{prop}(Support Verification $\texttt{FindAMEC}$)\label{prop:support_verification}
Under the event $\mathcal{E}$ and Assumption \ref{assump:lower_bound}, if $n = \phi_{\texttt{FindAMEC}}(\beta) = \frac{5}{\beta} \log(\frac{100 |\SState|^2 |\A|}{\beta^2 \delta} ) =  \tilde{\mathcal{O}}(\frac{1}{\beta})$ samples are collected for each state-action pair $(s,a) \in \SState \times \A$ then the support of $P$ is verified:
\begin{equation*}
    P(s,a,s') = \begin{cases}
    0, & \widehat{P}(s,a,s') = 0 \\
    1, & \widehat{P}(s,a,s') = 1 \\
    \in [\beta, 1-\beta], & \text{otherwise}
    \end{cases}
\end{equation*}
\end{prop}

\begin{proof}[\textbf{Proof} of Prop \ref{prop:support_verification}
]

Fix $(s,a,s') \in \SState \times \A \times \SState$. Suppose $\widehat{P}(s,a,s') \in \{0,1\}$ then by $\mathcal{E}$ we have 
\begin{equation}
    \frac{7 \log(4 n^2 |\SState|^2 |\A|/\delta)}{3(n-1)} \leq \beta
\end{equation}
Following the second half of the proof of \ref{lem:samples_req} with $\zeta^2 = \frac{3 \beta}{7},$ we have that if we take $n = \phi_{\texttt{FindAMEC}}(\beta) = \frac{5}{\beta} \log(\frac{100 |\SState^2 \A|}{\beta^2 \delta} ) > \frac{14}{3\beta} \log(\frac{784 |\SState^2 \A|}{ 9\beta^2 \delta} )$ then we have 
\begin{equation} \label{eq:samp_small_enough}
    |P(s,a,s') - \widehat{P}(s,a,s')| < \beta
\end{equation}

\textbf{Case $\widehat{P}(s,a,s') = 1$.} Suppose $\widehat{P}(s,a,s') = 1$. By Eq \eqref{eq:samp_small_enough}, $P(s,a,s') > 1- \beta$. By Assumption \ref{assump:lower_bound} together with the fact that $\sum_{x \in \SState} P(s,a,x) = 1$ then $P(s,a,x) = 0$ for any $x \neq s'$. Therefore, $P(s,a,s') = \widehat{P}(s,a,s')= 1$.

\textbf{Case $\widehat{P}(s,a,s') = 0$.} Suppose $\widehat{P}(s,a,s') = 0$. By Eq \eqref{eq:samp_small_enough}, $P(s,a,s') < \beta$. Hence $P(s,a,s') = \widehat{P}(s,a,s') = 0$, otherwise violating Assumption \ref{assump:lower_bound}.

\textbf{Case, Otherwise.} If $P(s,a,s') > 1-\beta$ or $P(s,a,s') < \beta$ then by following the above arguments we'd yield similar contradictions with Assumption \ref{assump:lower_bound}. Hence, $P(s,a,s') \in [\beta, 1-\beta]$
\end{proof}

\newpage
\subsection{\texttt{PlanRecurrent} proofs}

\gainConvCorrectness*
We formalize Prop \ref{prop:plan_recurrent} as follows by adding the necessary PAC statements:
\begin{restatable}[\texttt{PR} Convergence \& Correctness, Formal]{prop}{gainConvCorrectnessFormal}
\label{prop:plan_recurrent_formal} 
Let $\pi_A$ be the gain-optimal policy in AMEC $(A, \A)$. Algorithm \ref{algo:plan_recurrent} terminates after at most $\log_2\left(\frac{6|A|c_{\max}}{\epsilon_{\texttt{PR}} (1- \bar{\Delta}_A)}\right)$ repeats, and collects at most $n = \tilde{\mathcal{O}}(\frac{|A|^2c_{\max}^2}{\epsilon_{\texttt{PR}}^2 (1- \bar{\Delta}_A)^2})$ samples for each $(s,a) \in (A, \A_A)$. Under the event $\mathcal{E}$ and Assumption \ref{assump:lower_bound} then with probability $1-\delta$, the $\eta$-greedy policy $\pi$ w.r.t.~$v'$ (Alg.~\ref{algo:plan_recurrent}, Line 5) is gain optimal and probability optimal:
 $   |g_{\pi} - g_{\pi_A} | < \epsilon_{\texttt{PR}},$  $\Prob{\pi \models \varphi | s_0 \in A} = 1.$
\end{restatable}
\begin{proof}[Proof of Prop \ref{prop:plan_recurrent} \& Prop \ref{prop:plan_recurrent_formal}]
Let $\pi_{v'}$ be the greedy policy with respect to $v'$. Let $g_{\tilde{\pi}_A}^{\tilde{P}}$ be the gain of the gain-optimal policy, $\tilde{\pi}_A$, in $A$ with respect to dynamics $\tilde{P}$.

For the approximation error,
\begin{align*}
    0 \leq g^P_{\pi} - g^P_{\pi_A} &= g^P_{\pi} - g^{\tilde{P}}_{\pi} + g^{\tilde{P}}_{\pi} -
    g^{\tilde{P}}_{\pi_{v'}} + g^{\tilde{P}}_{\pi_{v'}} -
    g^{\tilde{P}}_{\ast} + g^{\tilde{P}}_{\ast} 
    - g^P_{\pi_A} \\
    &\leq \underbrace{|g^P_{\pi} - g^{\tilde{P}}_{\pi}|}_{(a)} + \underbrace{|g^{\tilde{P}}_{\pi} -
    g^{\tilde{P}}_{\pi_{v'}}|}_{(b)} + \underbrace{|g^{\tilde{P}}_{\pi_{v'}} -
    g^{\tilde{P}}_{\tilde{\pi}_A}|}_{(c)} + \underbrace{g^{\tilde{P}}_{\tilde{\pi}_A} 
    - g^P_{\pi_A}}_{(d)} \\
    &< \frac{\epsilon_{\texttt{PR}}}{3} + \frac{\epsilon_{\texttt{PR}}}{3} + \frac{\epsilon_{\texttt{PR}}}{3} + 0 = \epsilon_{\texttt{PR}}
\end{align*}

We have the first inequality because $\pi_A$ is gain optimal in $P$. By the Simulation Lemma \ref{lem:simulation_avg} we
have that $(a) < \frac{\epsilon_{\texttt{PR}}}{3}$ by setting $\epsilon_{(2)} = \frac{\epsilon_{\texttt{PR}}}{3}$ in the 
Lemma. By the $\eta$-greedy approximation Lemma \ref{lem:eta_greedy_approx} we have $(b) < \frac{\epsilon_{\texttt{PR}}}{3}$ by setting $\epsilon_{(1)} = \frac{\epsilon_{\texttt{PR}}}{3}$ in the Lemma. For $(c)$, since $\pi_{v'}$ represents the approximately optimal policy in $\tilde{P}$ then, by value iteration 
approximation guarantees, $(c) = |g^{\tilde{P}}_{\pi_{v'}} - g^{\tilde{P}}_{\tilde{\pi}_A}| < \frac{\epsilon^\Bellman_{\texttt{PR}}}{2} \leq \frac{\epsilon_{\texttt{PR}}}{3}$ by setting $\epsilon^\Bellman_{\texttt{PR}} = \frac{2 \epsilon_{\texttt{PR}}}{3}$ \cite{FruitUCRL22020}. It is known that, by optimism and the aperiodicity transformation \cite{FruitUCRL22020,Jaksch2010} for the average cost Bellman operator, $g^{\tilde{P}}_{\tilde{\pi}_A} < g^P_{\pi_A}$ implying $(d) < 0$. 

For the probability of satisfaction, when $s_0 \in A$, following a policy that samples every action in $\A_A$ with positive probability makes the markov chain $P_{\pi}$ recurrent. Thus, each $s \in A$ is visited infinitely often. In particular there is some $s^\ast \in A$ visited infinitely often, implying $\pi \models \varphi$. 

Convergence is guaranteed by Lemma \ref{lem:simulation_avg}: since $\rho$ is halved every iteration then $\rho$ never falls below $\frac{\epsilon_{(2)} (1-\bar{\Delta}_A)}{2 |A| c_{\max}}$, which is reached after $\log_{\frac{1}{2}}(\frac{\epsilon_{\texttt{PR}} (1-\bar{\Delta}_A)}{6 |A| c_{\max}}) = \log_{2}(\frac{6 |A| c_{\max}}{\epsilon_{\texttt{PR}} (1-\bar{\Delta}_A)})$ iterations (since $\epsilon_{(2)} = \frac{\epsilon_{\texttt{PR}}}{3}$). Further by Lemma \ref{lem:simulation_avg}, we get the sample complexity $n = \tilde{\mathcal{O}}(\frac{|A|^2 c_{\max}^2}{\epsilon_{\texttt{PR}}^2 (1- \bar{\Delta}_A)^2})$, completing the proof.
\end{proof}

\begin{cor}\label{cor:plan_recurrent_approx}
Under the same assumptions as Prop \ref{prop:plan_recurrent_formal}, in addition, 
$|g_{\pi}^P - \widehat{g}_{\tilde{\pi}_A}^{\tilde{P}}| \leq \frac{4\epsilon_{\texttt{PR}}}{3}$.
\end{cor}
\begin{proof}
Continuing the same argument as in Prop \ref{prop:plan_recurrent_formal}, we have 
\begin{align*}
    0 \leq g_{\pi}^P - g_{\tilde{\pi}_{A}}^{\tilde{P}} + g_{\tilde{\pi}_{A}}^{\tilde{P}} - \widehat{g}_{\tilde{\pi}_{A}}^{\tilde{P}} \leq \epsilon_{\texttt{PR}} + \frac{\epsilon^{\Bellman}_{\texttt{PR}}}{2} = \frac{4\epsilon_{\texttt{PR}}}{3}
\end{align*}
where we use triangle inequality and appeal to Prop \ref{prop:plan_recurrent_formal} for $|g_{\pi}^P - g_{\tilde{\pi}_{A}}^{\tilde{P}}| \leq \epsilon_{\texttt{PR}}$ and \cite{FruitUCRL22020} where $|g_{\tilde{\pi}_{A}}^{\tilde{P}} - \widehat{g}_{\pi_A}^{\tilde{P}}| \leq \frac{\epsilon^{\Bellman}_{\texttt{PR}}}{2} \leq \frac{\epsilon_{\texttt{PR}}}{3}$ since $\epsilon^{\Bellman}_{\texttt{PR}} = \frac{2 \epsilon_{\texttt{PR}}}{3}$.
\end{proof}

\begin{lem}($\eta$-greedy approximation)\label{lem:eta_greedy_approx} Let $P$ be any dynamics. Let $\pi$ be a greedy policy in AMEC $(A, \A_A)$ with dynamics $P$. With $0 \leq \eta \leq 1$, let $\pi_{\eta}$ be $\eta$-greedy with respect to $\pi$. Then, for any error $\epsilon_{(1)} > 0$, there exists some threshold $\eta^\ast \in (0,1]$ such that when $\eta \in (0, \eta^\ast]$ we have
\begin{equation}
    |g_{\pi}^P - g^P_{\pi_\eta}| \leq \epsilon_{(1)}
\end{equation}
\end{lem}
\begin{proof}
Let $s_0,\ldots,s_{|A|-1}$ be any ordering of the states in $A$. The standard (non-optimistic) average cost Bellman equation with known dynamics $P$ is given by $\Bellman v(s) = g(s) + \min_{a \in \A_A(s)} \big(\mathcal{C}(s,a) +  P(s,a) v \big)$ for each $s \in A$ for a unique $g$ and $v$ unique up to a constant translation \cite{bertsekas2011dynamic}. Furthermore, since the end components are communicating sets then we know that $g$ is a constant vector, i.e. $g = g(s) = g(s')$ for any $s, s' \in A$ \cite{bertsekas2011dynamic}. Since $v$ is unique up to translation, we can always set $v(0) = 0$ to make $v$ unique. The evaluation equations, under policy $\pi$, is similarly, $\Bellman_{\pi} v(s) = g_{\pi} + \mathbb{E}_{a \sim \pi} \big[\mathcal{C}(s,a)] +  P_{\pi}(s,a) v $ \cite{bertsekas2011dynamic}. For more generality, instead of $P_{\pi}$ we consider $\alpha P_{\pi} + (1-\alpha)I$, an aperiodicity transform with any coefficient $\alpha \in [0, 1]$. Then the $\Bellman_{\pi}$ written as a system takes the form:
\begin{equation*}
    0_{2 |A|} = \underbrace{\left[\begin{array}{@{}c|c@{}}
          \alpha P_{\pi}-(1-\alpha)I
          & -I  \\
          \hline
          C & D
    \end{array}
    \right]}_{X_{\pi}}
    \underbrace{\begin{bmatrix} v_0 \\ \vdots \\ v_{|A|-1} \\ \hline g_0 \\ \vdots \\ g_{|A|-1} 
    \end{bmatrix}}_{y} - \underbrace{\begin{bmatrix} \E_{a \sim \pi}{\mathcal{C}(s_0, a)} \\ \vdots \\ \E_{a \sim \pi} {\mathcal{C}(s_{|A|-1}, a)} \\ 0 \\ \vdots \\ 0
    \end{bmatrix}}_{b_\pi}
\end{equation*}
with 
\begin{equation*}
C = \begin{bmatrix}
    1 &  0 & \ldots  & \\
    0 &  0 & \ldots  \\
    \vdots & \ddots & \\
    \end{bmatrix}, D = \begin{bmatrix}
    0 & \ldots &   & 0 \\
    1 & -1 &   & 0 \\
     & \ddots & \ddots  & \\
    0  &        & 1      & -1 \\
    \end{bmatrix}
\end{equation*}
This system combines $\Bellman v(s) = \E_{a \sim \pi(s)}[\mathcal{C}(s,a)] + (\alpha P_{\pi} + (1-\alpha)I)v$ together with $g(s) = g(s')$ for any $s, s' \in A$ and $v(0) = 0$.

Succinctly, $X_{\pi} y - b_\pi = 0$. Similarly, we have $X_{\pi_\eta} y' - b_{\pi_{\eta}} = 0$. Let $dX = X_{\pi_\eta} - X_{\pi}$, $db =  b_{\pi_\eta} - b_{\pi}$ and $dy = y' - y$ then $(X_{\pi} + dX)(y + dy) - (b_{\pi} + db) = 0$. Hence, 
\begin{align*}
    dy &= (X_{\pi} + dX)^{-1}(db - dX y) \\
    &= (I + X_{\pi}^{-1} dX)^{-1} X_{\pi}^{-1} (db - dX y) \\
\end{align*}
We calculate $\|dX\|_\infty:$
\begin{align*}
    \|dX\|_\infty &= \max_{s \in A} \sum_{s' \in A} |\alpha P_{\pi_\eta}(s,s') - \alpha P_{\pi}(s, s')| \\
    &= \max_{s \in A} \sum_{s' \in A} | \alpha ((1-\eta)P_{\pi}(s,s') + \eta P_{Unif}(s,s') ) - \alpha P_{\pi}(s,s') | \\
    &= \alpha \eta \max_{s \in A} \sum_{s' \in A} |  P_{Unif}(s,s') - P_{\pi}(s,s') | \\
    &\leq \alpha \eta 2 |A| \\
\end{align*}
By a similar argument, together with $\mathcal{C} \leq c_{\max}$, then $\|db\|_\infty \leq 2 \eta c_{\max}$
Hence,
\begin{align*}
    \|dy\|_\infty &\leq \|(I + X_{\pi}^{-1} dX)^{-1}\|_\infty \|X_{\pi}^{-1}\|_\infty (\|db\|_\infty + \|dX\|_\infty \|y\|_\infty) \\
    &\leq \frac{\|X_{\pi}^{-1}\|_\infty}{1 - \|X_{\pi}^{-1}\|_\infty \|dX\|_\infty} (\|db\|_\infty + \|dX\|_\infty \|y\|_\infty) \\
    &\leq \frac{\eta \|X_{\pi}^{-1}\|_\infty}{1 -  2 \alpha |A| \eta \|X_{\pi}^{-1}\|_\infty } ( 2 c_{\max} + 2 \alpha |A| \|y\|_\infty) \\
\end{align*}
By selecting 
\begin{equation*}
    \eta \leq \eta^\ast = \frac{\epsilon_{(1)}}{\|X_{\pi}^{-1}\|_\infty (2c_{\max} + 2 \alpha |A| \|y\|_\infty) + \epsilon_{(1)} 2 \alpha |A|\|X_{\pi}^{-1}\|_\infty }
\end{equation*}
we get that $\|dy\|_\infty \leq \epsilon_{(1)}$ and therefore $|g_{\pi}^P - g^P_{\pi_\eta}| \leq \epsilon_{(1)}$, as desired.
\end{proof}


\newpage
\begin{lem}(Simulation Lemma, Avg. Cost)\label{lem:simulation_avg} Fix some $\alpha \in (0,1)$ arbitrary. Let $\tilde{P}$ be the optimistic dynamics achieving the inner minimum of the Bellman equation with respect to $\Bellman^\alpha_{\texttt{PR}}$ (see Table \ref{tab:param}) in the AMEC given by $(A, \A_A)$. Let $\pi$ be the $\eta^\ast$ stochastic policy as in Lemma \ref{lem:eta_greedy_approx}. For some error $\epsilon_{(2)} > 0$. Let $m \in \mathbb{N}$ be the smallest value such that $\Delta((\alpha \tilde{P}_{\pi} + (1-\alpha) I)^m) < 1$. When $n$ is large enough that $\psi(n) \leq \frac{1}{\alpha^2 }\left(\left(\frac{\epsilon_{(2)} (1 - \Delta(\tilde{P}_{\alpha, \pi}^m))}{|A| c_{\max}} + 1 \right)^{1/m} - 1\right)$
then
\begin{equation}\label{eq:sim_lemma}
    |g_{\pi}^P - g^{\tilde{P}}_{\pi}| < \epsilon_{(2)}.
\end{equation}
Let $m = \max_{\pi \in \Pi_A} \min_{m \in \mathbb{N}} \{m | \Delta((\alpha P_{\pi_{\eta}^\ast} + (1-\alpha) I)^m) < 1 \}$ and $\bar{\Delta}_A = \max_{\pi \in \Pi_A} \Delta((\alpha P_{\pi_{\eta}^\ast} + (1-\alpha) I)^m)$ for  $\Pi_A$, the set of deterministic policies in $A$. Then, in particular, \eqref{eq:sim_lemma} holds after $n = \tilde{\mathcal{O}}(\frac{ |A|^{2} c_{\max}^2}{\epsilon_{(2)}^{2} (1- \bar{\Delta}_A)^{2}})$ samples are collected for each state-action pair in $(A, \A_A)$.
\end{lem}

\begin{proof} 
Consider, notationally, $P_{\alpha}(s,a,s') = \alpha P(s,a,s') + (1-\alpha) \mathbf{1}_{\{s = s'\}}$ be an aperiodicity transform with $\alpha \in (0,1)$. When fixed by a policy, then $P_{\alpha, \pi} = \alpha P_{\pi} + (1-\alpha) I$. By \cite{puterman2014markov} (Prop. 8.5.8), aperiodicity transforms do not affect gain. Hence $g_{\pi}^{P} = g_{\pi}^{P_{\alpha}}$ and $g_{\pi}^{\tilde{P}} = g_{\pi}^{\tilde{P}_{\alpha}}$. 
Let $x_{\pi, P_{\alpha}}$ be the stationary distribution of $\pi$ in $P_{\alpha}$ and $x_{\pi, \tilde{P}_{\alpha}}$ be the stationary distribution of $\pi$ in $\tilde{P}_{\alpha}$. These quantities exist due to the fact that $\pi$ has full support over $\A_{A}$ making both $P_{\alpha}, \tilde{P}_{\alpha}$ ergodic (finite, irreducible, recurrent, and aperiodic).
Hence,
\begin{align*}
    |g^P_{\pi} - g^{\tilde{P}}_{\pi}| &= |g^{P_{\alpha}}_{\pi} - g^{\tilde{P}_{\alpha}}_{\pi}| \\
    &= |\mathbb{E}_{s \sim x_{\pi, P_{\alpha}}}[\mathbb{E}_{a \sim \pi(s)}[\mathcal{C}(s,a)]] -  \mathbb{E}_{s \sim x_{\pi, \tilde{P}_{\alpha}}}[\mathbb{E}_{a \sim \pi(s)}[\mathcal{C}(s,a)]]| \\
    &= |\sum_{s \in A} \mathbb{E}_{a \sim \pi(s)}[\mathcal{C}(s,a)] (x_{\pi, P_{\alpha}}(s) - x_{\pi, \tilde{P}_{\alpha}}(s))| \\
    &\leq c_{\max} \| x_{\pi, P_{\alpha}} - x_{\pi, \tilde{P}_{\alpha}} \|_1 
\end{align*}
To bound $\| x_{\pi, P_{\alpha}} - x_{\pi, \tilde{P}_{\alpha}} \|_1$, we appeal to classic stationary-distribution perturbation bounds \cite{Cho_Meyer_2001}. First, since $\tilde{P}_{\alpha, \pi}$ is ergodic then $\exists m_0 < \infty$ such that for any $m \geq m_0$ then $\Delta(\tilde{P}_{\alpha, \pi}^m) < 1.$ Then, in particular, $\| x_{\pi, P_{\alpha}} - x_{\pi, \tilde{P}_{\alpha}} \|_{1} \leq \frac{\| \tilde{P}_{\alpha, \pi}^m - P_{\alpha, \pi}^m \|_\infty }{1- \Delta(\tilde{P}_{\alpha, \pi}^m)}$ \cite{seneta_1988, Cho_Meyer_2001}. Let $E = P_{\pi, \alpha} - \tilde{P}_{\pi, \alpha}$, and thus $\|E\|_\infty = \alpha \| P_{\pi} - \tilde{P}_{\pi} \|_\infty \leq \alpha |A| \psi(n)$. Then,
\begin{align*}
    \| \tilde{P}_{\alpha, \pi}^m - P_{\alpha, \pi}^m \|_\infty &= \| \tilde{P}_{\alpha, \pi}^m  - (\alpha P_{\pi} + (1-\alpha)I)^m \|_\infty \\
    &= \| \tilde{P}_{\alpha, \pi}^m - (\alpha E + \alpha \tilde{P}_{\pi} + (1-\alpha)I)^m \|_\infty \\
    &= \|  \tilde{P}_{\alpha, \pi}^m - (\alpha E + \tilde{P}_{\alpha, \pi})^m \|_\infty \\
    &\leq (\alpha \|E\|_\infty + 1)^m - 1 \\
    &\leq (\alpha^2 |A| \psi(n) + 1)^m - 1
\end{align*}
where in the second-to-last inequality uses that $\|\tilde{P}_{\alpha, \pi}\|_\infty = 1$ and $\|AB\|_\infty \leq \|A\|_\infty \|B\|_\infty$ for matrices $A,B$. Putting it all together we have that
\begin{equation}
    |g^P_{\pi} - g^{\tilde{P}}_{\pi}| \leq c_{\max} \frac{(\alpha^2 |A| \psi(n) + 1)^m - 1 }{1 - \Delta(\tilde{P}_{\alpha, \pi}^m)} 
\end{equation}
We therefore require that 
\begin{equation} \label{eq:psi_sim_lemma}
    \psi(n) \leq \frac{1}{\alpha^2 |A|}\left(\left(\frac{\epsilon_{(2)} (1 - \Delta(\tilde{P}_{\alpha, \pi}^m))}{c_{\max}} + 1 \right)^{1/m} - 1\right)
\end{equation}
to yield $|g^P_{\pi} - g^{\tilde{P}}_{\pi}| < \epsilon_{(2)}$. The equation \eqref{eq:psi_sim_lemma} also holds with $\tilde{P}$ replaced with $P$, with (some other) $m$ appropriate.

In the AMEC $(A, \A_A)$ then there are at most $|\Pi_A| = |A|^{|\A_A|}$ deterministic policies. For each policy $\pi \in \Pi_A$, there is some $\eta_\pi^\ast$ satisfying Lemma $\ref{lem:eta_greedy_approx}$. Let $m = \max_{\pi \in \Pi_A} \min_{m \in \mathbb{N}} \{m | \Delta(P_{\alpha, \pi_{\eta^\ast_\pi}}^m) < 1 \}$ and $\bar{\Delta}_A = \max_{\pi \in \Pi_A} \Delta(P_{\alpha, \pi_{ \eta^\ast_\pi}}^m) < 1$ (recall this is guaranteed because $P_{\alpha, \pi_{ \eta^\ast_\pi}}$ is ergodic). Then, when $\psi(n) < \frac{1}{\alpha^2 |A|}\left(\left(\frac{\epsilon_{(2)} (1 - \bar{\Delta}_A)}{ c_{\max}} + 1 \right)^{1/m} - 1\right)$ then $|g^P_{\pi} - g^{\tilde{P}}_{\pi}| < \epsilon_{((2)}$. By Lemma \ref{lem:samples_req}, we have $n = \mathcal{\tilde{O}}(\frac{ |A|^{ \frac{2}{m}} c_{\max}^{\frac{2}{m} }}{\epsilon_{(2)}^{\frac{2}{m}} (1- \bar{\Delta}_A)^{\frac{2}{m}}}) = \mathcal{\tilde{O}}(\frac{ |A|^{2} c_{\max}^{2}}{\epsilon_{(2)}^{2} (1- \bar{\Delta}_A)^{2}})$,
since $m =1$ achieves the maximum.
\end{proof}

\begin{rem}
We do not require knowledge of $\bar{\Delta}_A < 1$. The existence is sufficient to guarantee convergence.
\end{rem}
\begin{rem}
The function $\Delta(M)$, coefficient of ergodicity of matrix $M$, is a measure (and bound) of the second largest eigenvalue of $M$.
\end{rem}
\begin{rem}
In the main paper, we assume that $m=1$ and $\alpha=1$, for simplicity in exposition. For full rigor, $m$ may be larger, though typically small. $m$ can be seen as the smallest value making any column of $P_{\alpha, \pi}^m$ dense. 
From a computational perspective, it is efficient to compute powers of $\tilde{P}_{\alpha, \pi}$ and stop when $\tilde{P}_{\alpha, \pi}^m$ has a dense column, making $\Delta(\tilde{P}_{\alpha, \pi}) < 1$. From there, we can check if $\rho$ (Line 6, Algo \ref{algo:plan_recurrent}) satisfies the r.h.s of Eq \eqref{eq:psi_sim_lemma}. We present the samples required by maximizing over $m \in \mathbb{N}$.
\end{rem}


\newpage
\subsection{\texttt{PlanTransient} proofs}

\transientConvCorrectness*
\begin{restatable}[\texttt{PlanTransient} Convergence \& Correctness, Formal]{prop}{transientConvCorrectnessFormal}
\label{prop:plan_transient_formal}
Let $\{A_i, g_i\}_{i=1}^k$ be the set of inputs to Algorithm \ref{algo:plan_transient}, together with error $\epsilon_{\texttt{PT}} > 0$. Denote the cost- and prob-optimal policy as $\pi'$. After collecting at most $n = \tilde{\mathcal{O}}(\frac{|\SState \setminus \cup_{i=1}^k A_i|^2 \bar{V}^4}{c_{\min}^2 
\epsilon_{\texttt{PT}}^2 
\epsilon_{\varphi}^4})$ samples for each $(s,a) \in (\SState \setminus \cup_{i=1}^k A_i) \times \A$, under the event $\mathcal{E}$ and Assumption \ref{assump:lower_bound} then with probability $1-\delta$, , the greedy policy $\pi$ w.r.t.~$v'$ (Alg.~\ref{algo:plan_transient}, Line 3) is both cost and probability optimal:
\begin{equation*}
    \|\tilde{V}_{\pi} - \tilde{V}_{\pi'} \| < \epsilon_{\texttt{PT}}, \quad |\Prob{\pi \text{ reaches } \cup_{i=1}^k A_i} - \Prob{\pi' \text{ reaches } \cup_{i=1}^k A_i}| \leq \epsilon_{\varphi}.
\end{equation*}
\end{restatable}
\begin{proof}[Proof of \ref{prop:plan_transient}]
Convergence follows from boundedness of $\|v\| \leq \bar{V}$, and monotone convergence and is well studied \cite{puterman2014markov, Jaksch2010, TarbouriechSSPMinimax, FruitUCRL22020}.

Fix $\lambda > 0$ and drop it from the notation $V^P_{\pi, \lambda}$. Let $\tilde{V}_{\ast}^{\tilde{P}}$ be the value function for the optimal policy in $\tilde{P}$. For the approximation error, we have
\begin{align*}
0 \leq \tilde{V}^{P}_{\pi} - \tilde{V}^{P}_{\ast} &= \underbrace{\tilde{V}^{P}_{\pi} - \tilde{V}^{\tilde{P}}_{\pi}}_{(a)} + \underbrace{\tilde{V}^{\tilde{P}}_{\pi} - \tilde{V}^{P}_{\ast}}_{(b)} < \epsilon_{\texttt{PT}} 
\end{align*}
For $(a)$ we appeal to Lemma $\ref{lem:simulation_transient}$ and set  $\epsilon_{(3)} = \epsilon_{\texttt{PT}}/2$ requiring that $\psi(n) = \frac{ \epsilon_{\texttt{PT}} c_{\min} }{14 |\SState \setminus \cup_{i=1}^k A_i| \bar{V}^2 (1+\frac{1}{\epsilon_{\varphi}})^2 }$, occuring when $n = \tilde{\mathcal{O}}((\frac{|\SState \setminus \cup_{i=1}^k A_i| \bar{V}^2}{ \epsilon_{\texttt{PT}} \epsilon_{\varphi}^2 c_{\min}})^{2})$ samples per state-action pair have been collected. For $(b)$, by Lemma \ref{lem:EVI_bound}, by selecting $\epsilon^\Bellman_{\texttt{PT}} = \frac{ c_{\min} \epsilon_{\texttt{PT}} \epsilon_{\varphi}}{4 \bar{V}}$ we have that 
\begin{align*}
     V^{\tilde{P}}_{\pi} - V^P_{\ast} &\leq (1 + \frac{2 \epsilon^\Bellman_{\texttt{PT}} }{c_{\min}}) v - V^P_{\ast} \\
     &= \frac{2 \epsilon^\Bellman_{\texttt{PT}}  v}{c_{\min}} \\
     &\leq \frac{2 \bar{V} \epsilon^\Bellman_{\texttt{PT}}  }{ \epsilon_{\varphi} c_{\min}} \leq \frac{\epsilon_{\texttt{PT}}}{2}.
\end{align*} 

For the probability of satisfaction, by Prop \ref{prop:select_V}, we have that $\pi$ and $\pi^\ast$ coincide in probability of reaching the states in $\cup_{i=1}^k A_i$.
\end{proof}

\begin{prop}[Selecting a bound on $\|v\|$]\label{prop:select_V} Let $\{A_i, g_i\}_{i=1}^k$ be the set of inputs to Algorithm \ref{algo:plan_transient}. Let  $\pi'$ have maximal probability of reaching $\cup_{i=1}^k A_i$. Then, with error $\epsilon_{\varphi}>0$, bounding $\|v\|_\infty = \| \Bellman_{\texttt{PT}}v\|_\infty \leq \frac{\bar{V}}{\epsilon_{\varphi}}$ where $\bar{V} \geq \left( \frac{1}{\beta^{|\SState|}} \left( \frac{1-\beta^{|\SState|}}{1-\beta} \right) + \lambda\right) c_{\max} $ guarantees that $\pi$ returned by Algorithm \ref{algo:plan_transient} is near probability optimal: 
\begin{equation*}
    |\Prob{\pi \models \varphi} - \Prob{ \pi' \models \varphi} | < \epsilon_{\varphi}
\end{equation*}
\end{prop}
\begin{proof}[Proof of \ref{prop:select_V}]
Suppose $\bar{V} \geq J_{\pi} + \lambda c_{\max}$ for any $\pi \in \Pi$. Let $\frac{\bar{V}}{\epsilon_{\varphi}}$ be chosen as upper bound on $\|v\| = \|\Bellman_{\texttt{PT}} v\|$. Denote $\Prob{\pi \models \varphi}$ as $p$, and $\Prob{\pi' \models \varphi}$ as $p^\ast$. Suppose, for contradiction, $p^\ast - p > \epsilon_{\varphi}$, yet $\pi$ is returned by the Algorithm. This would imply that $\tilde{V}_{\pi} \leq \tilde{V}_{\pi'}$.

Hence, 
\begin{align*}
    0 \leq \tilde{V}_{\pi'} - \tilde{V}_{\pi}  &\leq \underbrace{p^\ast(J_{\pi'} + \lambda \sum_{i=1}^k \frac{p^\ast_i}{p^\ast} \widehat{g}^{\tilde{P}}_{\tilde{\pi}_{A_i}})}_{ \leq J_{\pi} + \lambda c_{\max}} - \underbrace{ p (J_{\pi} + \lambda \sum_{i=1}^k \frac{p_i}{p} \widehat{g}^{\tilde{P}}_{\tilde{\pi}_{A_i}})}_{\geq 0} +  \underbrace{(p -p^\ast)}_{< -\epsilon_{\varphi}}\frac{\bar{V}}{\epsilon_{\varphi}} \\
    &< J_{\pi} + \lambda c_{\max} - \bar{V} \\
    &\leq 0
\end{align*}
Hence, we have a contradiction. Thus, $|p^\ast - p| \leq \epsilon_{\varphi}$ if $\bar{V} \geq J_{\pi} + \lambda c_{\max}$ for any $\pi \in \Pi$. In fact, since the solution to $\Bellman_{\texttt{PT}}$ is deterministic, it suffices to consider only deterministic $\Pi$.

We will now bound $J_{\pi} = \mathbb{E}_{\tau \sim \Tau_{\pi}}\left[\sum_{t=0}^{\kappa_{\pi}} \Cost(s_t, \pi(s_t)) \bigg| \tau \models \varphi \right] \leq c_{\max} \mathbb{E}_{\tau \sim \Tau_{\pi}}[\kappa_{\tau} | \tau \models \varphi]$, as this is the only unknown quantity. Here $\mathbb{E}_{\tau \sim \Tau_{\pi}}[\kappa_{\tau} | \tau \models \varphi]$ is the expected number of steps it takes $\pi$ to leave the transient states. This means that a worst-case bound would be a policy that remains in the transient states as long as possible. 

We construct the worst-case scenario and give a justification, a formal proof follows from induction. Suppose the starting state is $s_0$. If $\pi$ induces a prob-1 transition back to $s_0$ then $s_0$ is recurrent, and so $\kappa_{\tau}$ would be small. Instead, $\pi$ induces a prob $1-\beta$ transition to $s_0$ and a prob $\beta$ transition to $s_1$. Notice that the transition to $s_1$ must be at least probability $\beta$ due to Assumption \ref{assump:lower_bound}. Again, if $s_1$ gave all of its probability to $s_1$ or $s_0$ then a MEC would form and strictly decrease $\kappa_{\tau}$. This process repeats until we reach state $s_{|\SState|-1}$, which has to have a self-loop. If it does not, then, again a large MEC would form and decrease $\kappa_{\tau}$. Of course, this is the well known chain graph, with easily computable expected hitting time: $\mathbb{E}_{\tau \sim \Tau_{\pi}}[\kappa_{\tau}] \leq \frac{1}{\beta^{|\SState|}} \frac{1-\beta^{|\SState|}}{1-\beta}$. By making $s_{|\SState|-1}$ the accepting state, then $\mathbb{E}_{\tau \sim \Tau_{\pi}}[\kappa_{\tau} | \tau \models \varphi] = \mathbb{E}_{\tau \sim \Tau_{\pi}}[\kappa_{\tau}] = \frac{1}{\beta^{|\SState|}} \frac{1-\beta^{|\SState|}}{1-\beta}$ achieves the bound. Any other choice of accepting states would strictly decrease $\kappa_{\tau}$. Hence, we can select
\begin{equation*}
    \bar{V} \geq \left( \frac{1}{\beta^{|\SState|}} \left( \frac{1-\beta^{|\SState|}}{1-\beta} \right) + \lambda\right) c_{\max} \geq J_{\pi} + \lambda c_{\max},
\end{equation*}
completing the proof.
\end{proof}

\begin{rem}
It may also be possible to empirically estimate $J_{\pi}$ rather than take the bound from Prop \ref{prop:select_V}, considering that we have the structure of $P$ through $\widehat{P}$. We give the high level idea. We know all of the AMECs and rejecting EC, so we have all the transient states (denoted $T$). Then for some policy $\pi$ and $P' \in \mathcal{P}$, submatrix $Q_{\pi}(s,s') = P'_{\pi}(s,s')$ for $s,s' \in T$ represents the transitions in the transient states. It is well known that $\mathbb{E}_{\tau \sim \Tau_{\pi}}[\kappa_{\tau}] = \|(I-Q)^{-1}\|_\infty$. Taking the max over all $\pi \in \Pi$, $P' \in \mathcal{P}$, and finally multiplying by $c_{\max}$ gives a bound on $J_{\pi}$.
\end{rem}

\begin{lem}\label{lem:simulation_transient}(Simulation Lemma, Transient Cost \cite{Tarbouriech2021})
Consider an MDP $(\SState, \A, \ldots).$ For any two transition functions $P', P'' \in \mathcal{P}$, policy $\pi$, and error $\epsilon_{(3)} > 0$ then
\begin{equation*}
    \| \tilde{V}^{P''}_{\pi}\|_\infty = \| \tilde{V}^{P'}_{\pi}\|_\infty \leq (1+\frac{1}{\epsilon_{\varphi}})\bar{V}, \quad \| \tilde{V}^{P'}_{\pi} - \tilde{V}^{P''}_{\pi} \|_\infty  \leq \frac{7 |\SState| \bar{V}^2 (1+\frac{1}{\epsilon_{\varphi}})^2 \psi(n)}{c_{\min}} \leq \epsilon_{(3)}
\end{equation*}
occurring after $n = \tilde{\mathcal{O}}(\frac{|\SState|^2 \bar{V}^4}{\epsilon_{(3)}^2 \epsilon_{\varphi}^4 c_{\min}^2 })$ samples from each state-action pair in $\SState \times \A$.
\end{lem}

\begin{proof}
Direct consequence of the definition of $\bar{V}$ from Prop \ref{prop:select_V}, application of Lemma 2 from \cite{Tarbouriech2021} and Lemma \ref{lem:samples_req}.
\end{proof}

\begin{lem}\label{lem:EVI_bound}(EVI Bound, \cite{Tarbouriech2021}) Suppose $v$ is returned by $\texttt{VI}$ with accuracy $\epsilon^\Bellman_{\texttt{PT}}$ with Bellman equation $\Bellman_{\texttt{PT}}$ (See Table \ref{tab:param}). Suppose $\pi$ is greedy with respect to $v$. If $\epsilon^\Bellman_{\texttt{PT}} \leq \frac{c_{\min}}{2}$ then, element-wise,
\begin{equation*}
    v \leq \tilde{V}^P_{\pi^\ast},\quad v \leq \tilde{V}^{\tilde{P}}_{\pi}
    \leq (1 +  \frac{2 \epsilon^\Bellman_{\texttt{PT}}}{c_{\min}}) 
    v
\end{equation*}
\end{lem}

\newpage
\section{Conjecture on Sample Complexity}\label{app:conjecture}

As we have proven in Theorem \ref{thm:generalization}, the optimal policy creates a set of AMECs which coincide with ${(A_i, \A_i)}_{i=1}^k$. For any potential AMEC, we need to guarantee probabilistic closure. For each state-action pair $(s,a) \in A_i \times \A_{A_i}$ we have to sample enough times to guarantee that we have ``collected'' all of the possible unique transitions $(s,a,s')$. Indeed, this is similar to the famous coupon collection problem, where we want to know how much time it will take to collect all unique transitions $(s,a,s')$. Suppose there are $m$ unique tuples each with probability $\beta = \frac{1}{m}$.

We can use a Chebyshev-based lower bound:
$$
\Prob{N > m \log m - \log(\frac{1}{\delta}) m } \geq \delta
$$
Simplifying, we get that $\Prob{N > m \log(\frac{m}{\delta}) } \geq \delta$. Thus, the number of transitions needed is 
$$
N = \Omega(m \log(\frac{m}{\delta})) = \Omega(\frac{1}{\beta} \log(\frac{1}{\beta\delta})) = \tilde{\Omega}(\frac{1}{\beta})
$$

Further, \cite{Wu2019Support} show that indeed $N \geq \frac{\beta}{\log{\beta}}$.
\newpage
\section{Additional Algorithms}\label{app:aditional_algorithms}

In this section we discuss the additional subroutines used in this paper. We discuss the case where selecting $\cup_{i=1}^k A_i$ as the terminal states for SSP in Algo \ref{algo:main} can fail and an alternative solution.

\subsection{Value Iteration}

Our version of Value Iteration $\texttt{VI}$ (Algo \ref{algo:VI}) is a two-in-one version, due to the similarity of Relative VI (used in $\texttt{PlanRecurrent}$) and SSP (used in $\texttt{PlanTransient}$). The general idea is that you apply the Bellman Operator $\Bellman$ onto your current iterate $v_n$ repeatedly until $d(v_{n+1}, v_n)$ exceeds $\epsilon$. When we wish to find the gain, then $V_T$ (terminal states) is empty, and we use shifting by the first value of $v_n(0)$ for stability \cite{bertsekas2011dynamic}. In other words, we subtract $v_n(0)$ from every value of $v_n$. On the other hand, if a set of terminal costs is provided then these represent the set of states that we want to reach through SSP and the value $v_n(s) = V_T(s)$ is known and must be kept fixed throughout applications of $\Bellman.$ The only difference in our application of $\Bellman$ over standard Bellman operators is that $\Bellman$ is optimistic and has an interior minimization over $\min_{p \in \mathcal{P}(s,a)} p^T v_n$ (See Table \ref{tab:param}). To solve this, minimization we use a modified version from \cite{Jaksch2010} given in Algo \ref{algo:inner_min_P}. The idea of Algo \ref{algo:inner_min_P} is simple: put all the mass of $\tilde{P}$ onto the lowest possible values of $v_n$ while still being consistent with $\widehat{P}$. This is efficient as it requires an ordering over $v$ and then a single pass over the states $\tilde{S}$. The calculated probability $p(\tilde{s}_l)$ (see Algo \ref{algo:inner_min_P}) are what we call the optimistic dynamics $\tilde{P}(s,a,\tilde{s}_l)$.

\begin{minipage}{1\textwidth}
\begin{algorithm}[H]
	\begin{small}
	\centering
	\caption{Value Iteration $(\texttt{VI})$} 
	\label{algo:VI}
	\begin{algorithmic}[1]
	    \REQUIRE Optimistic Bellman Operator $\Bellman$, Error Measure $d$, accuracy $\epsilon > 0$, $V_T$ terminal values (optional)
	    \STATE Set $n=0, v_0 = 0_S, v_1 = \Bellman v_0$
	    \REPEAT
	        \STATE $n \pp 1$
	        \IF{$V_T$ is empty}
	            \STATE Shift $v_n \leftarrow v_n - v_n(0)\mathbf{1} $ \hfill $\COMMENT{\text{Relative Value Iteration}}$
	        \ELSE
	            \STATE $v_n(s) \leftarrow V_T(s)$ for $s \in V_T$ \hfill$\COMMENT{\text{SSP}}$
	        \ENDIF
	        \STATE Apply operator $v_{n+1}, \tilde{P} \leftarrow \Bellman v_n $ \hfill$\COMMENT{\text{Bellman Backup}}$
	    \UNTIL{$d(v_{n+1}, v_n) > \epsilon$}
	\RETURN $v_{n+1}, v_n, \tilde{P}$
	\end{algorithmic}
	\end{small}
\end{algorithm}
\end{minipage}

\begin{minipage}{1\textwidth}
\begin{algorithm}[H]
	\begin{small}
	\caption{$\texttt{InnerMin}$ (for \texttt{PT}/\texttt{PR}) } 
	\label{algo:inner_min_P}
	\begin{algorithmic}[1]
	\REQUIRE A set of states $\tilde{\SState}$, current estimate from VI $v_n$, estimates $\widehat{P}(s,a,\cdot)$ for a specific $(s,a)$ pair with $s \in \tilde{\SState}$, errors $\psi(n)$, lower bound $\beta$ (See Assumption \ref{assump:lower_bound}) \\
	\STATE Sort $\tilde{\SState}= \{\tilde{s}_1, \tilde{s}_2, \ldots, \tilde{s}_m\}$ according to $v_n(\tilde{s}_1) \leq v_n(\tilde{s}_2) \leq \ldots \leq v_n(\tilde{s}_m)$, where $v_n$ is the current \\ 
	\STATE Set
	\begin{equation*}
	    p(\tilde{s}_1) = \begin{cases}
	                        \min(1-\beta, \widehat{P}(s,a, \tilde{s}_1) + \psi(n)),& \widehat{P}(s,a, \tilde{s}_1) \not\in \{0,1\} \\
	                        1,& \widehat{P}(s,a, \tilde{s}_1) = 1\\
	                        0,& \widehat{P}(s,a, \tilde{s}_1) = 0
	                     \end{cases}
	\end{equation*}
	\STATE For remaining $j > 1$, set $p(\tilde{s}_j) = \widehat{P}(s,a, \tilde{s}_j)$ \\
	\STATE Set $l \leftarrow m$
	\WHILE{$\sum_{\tilde{s}_j \in \tilde{\SState}} p(\tilde{s}_j)> 1$}
	    \STATE Reset 
	    \begin{equation*} 
    	    p(\tilde{s}_l) = \begin{cases}
    	                        \max(\beta, 1- \sum_{\tilde{s}_j \neq \tilde{s}_l} p(\tilde{s}_j)),& \widehat{P}(s,a, \tilde{s}_l) \not\in \{0,1\} \\
    	                        1,& \widehat{P}(s,a, \tilde{s}_l) = 1\\
    	                        0,& \widehat{P}(s,a, \tilde{s}_l) = 0
    	                     \end{cases}
	    \end{equation*}
	    \STATE Decrement $l \leftarrow l-1$
	\ENDWHILE
	\STATE Set $\tilde{P}(s,a,\tilde{s}) = p(\tilde{s})$ for each $\tilde{s} \in \tilde{S}$
	\RETURN $\tilde{P}(s,a,\tilde{s})$
	\end{algorithmic}
	\end{small}
\end{algorithm}
\end{minipage}

\subsection{Modified Algorithm handling Blocking Failure in Algorithm \ref{algo:main}} \label{app:sec:blocking}

\begin{figure}[!htp]
\centering
\includegraphics[width=.3\linewidth]{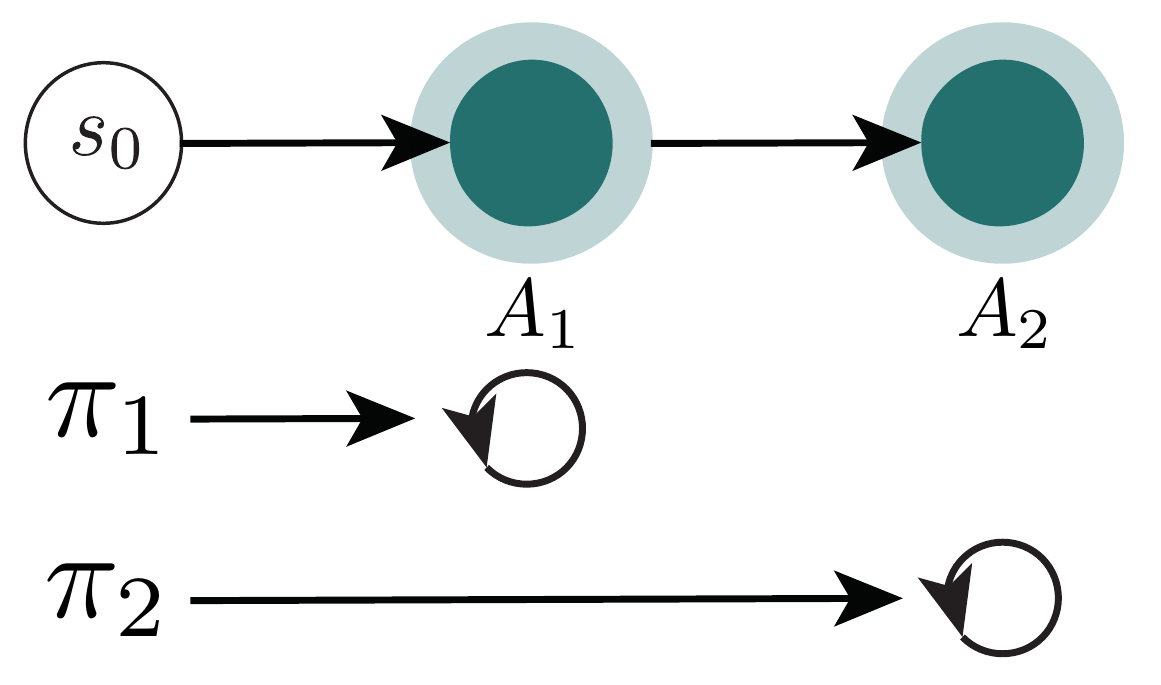}
\caption{\textit{Blocking Issue. If $A_1$ is included in the terminal AMECs (the states we want to reach) then once it is reached $\pi_{A_1}$ is instantiated and $A_1$ becomes recurrent, implying only $\pi_1$ is considered. However, even though it may be the case that $J_{\pi_1} < J_{\pi_2}$, we may still have $V_{\pi_1} > V_{\pi_2}$. This example demonstrates the necessity to pick the terminal AMECs properly, rather than just the union of all AMECs found, to avoid blocking.
}}
 \label{fig:blocking}
\end{figure}

One of the failure modes of Algorithm \ref{algo:main} is in its selection of which AMECs are the necessary AMECs to reach. In fact, by selecting unnecessary AMECs, the SSP procedure fails to treat some AMECs as transient states when in fact, maybe, lower cost could have been achieved if they were. One way to see this is to consider a single directional chain of AMECS (See Figure \ref{fig:blocking}). In the figure, two policies can be considered: (1) $\pi_1$ that reaches for $A_1$ and then starts $\pi_{A_1}$ when $A_1$ is reached, and (2) $\pi_2$ that reaches for $A_2$ and then starts $\pi_{A_2}$ when $A_2$ is reached. It may be the case that $V_{\pi_2} < V_{\pi_1}$ despite $J_{\pi_2} > J_{\pi_1}$, since it requires a longer cost path to reach the desired AMEC. Despite this observation, when $A_1$ is selected as terminal states in the subroutine $\texttt{PlanTransient}$ (Algo \ref{algo:plan_transient}), we disallow consideration of $\pi_2$ at all. As explained in the proof of Theorem \ref{thm:generalization}, whatever AMECs are induced by $\pi^\ast$ coincide with $AMEC = \{A_i, \A_{A_i}\}_{i=1}^k$. Let $\Omega = 2^{AMEC} \setminus \varnothing$, all non-empty subsets of AMECs (possible targets). Since all accepting trajectories of $\pi^\ast$ land in an AMEC, then another way of looking at $\pi^\ast$ is:
\begin{equation*}
    \pi^\ast = \min_{\omega \in \Omega} \min_{\pi \in \tilde{\Pi}(\omega)} V_{\pi} 
\end{equation*}
where $\tilde{\Pi}(\omega) = \{\pi \in \Pi_{\max} | \pi(s,a) = \pi_{A_i}(s,a) \text{ for } s \in A_i \in \omega, a \in \A_{A_i}(s) \}$, which is a policy class where the only degrees of freedom are outside of $\omega$. In other words, $\pi \in \tilde{\Pi}(\omega)$ is followed until the trajectory hits $A_i \in \omega$ and then $\pi_{A_i}$ is followed thereafter.

We will reconcile this failure mode of $\texttt{PlanTransient}$ through a modified, nonblocking, subroutine $\texttt{NoBlockPlanTransient}$ (Algo \ref{algo:no_block_plan_transient}). 

\begin{minipage}{\textwidth}
\begin{algorithm}[H]
	\begin{small}
	\caption{$\texttt{NoBlockPlanTransient}$ ($\texttt{NB-PT}$) } 
	\label{algo:no_block_plan_transient}
	\begin{algorithmic}[1]
	    \REQUIRE States \& gains: $\{(A_i, g_i)\}_{i=1}^k$, err. $\epsilon_{\texttt{PT}} > 0$
	    
	    \STATE Set $v(s) = \infty$ for each $s \in \SState$.
	    \STATE Sample $\phi_{\texttt{PT}}$ times $\forall (s,a)\in \SState \times \A$  
	    \FOR{$\omega \in 2^{\{A_i\}_{i=1}^k} \setminus \varnothing$}
	        \STATE Set $V_T(s) = \lambda g_i$ for $s \in A_i \subseteq \omega$ 
	        \STATE $v_{\omega}', v_{\omega}, \tilde{P}  \leftarrow \texttt{VI}(\Bellman_{\texttt{PT}}, d_{\texttt{PT}}, \epsilon^\Bellman_{\texttt{PT}}, V_T)$
	        \IF{$\mathbb{E}_{s \sim d_0}[v_{\omega}'(s)] < \mathbb{E}_{s \sim d_0}[v(s)]$}
	            \STATE Set $v = v_{\omega}'$
	        \ENDIF  
	    \ENDFOR
	\STATE Set $\pi \leftarrow $greedy policy w.r.t $v$
	\RETURN $\pi$
	\end{algorithmic}
	\end{small}
\end{algorithm}
\end{minipage}

The proof of correctness follows from the fact that $v'_\omega$ closely tracks $V_{\pi}$ where $\pi$ is greedy wrt $v'_\omega$. Then, selecting the smallest $V_{\pi}$ coincides with $V_{\pi^\ast}$.

\begin{prop}[Proof of Correctness and Convergence of \texttt{NoBlockPlanTransient}]\label{prop:no_block_plan_transient}
After collecting at most $n = \tilde{\mathcal{O}}(\frac{|\SState|^2 \bar{V}^4}{c_{\min}^2 
\epsilon_{\texttt{PT}}^2 
\epsilon_{\varphi}^4})$ samples for each $(s,a) \in \SState \times \A$, under the event $\mathcal{E}$ and Assumption \ref{assump:lower_bound} then with probability $1-\delta$, , the greedy policy $\pi$ w.r.t.~$v'$ (Alg.~\ref{algo:plan_transient}, Line 3) is both cost and probability optimal:
\begin{equation*}
    \|\tilde{V}_{\pi} - \tilde{V}_{\pi^\ast} \| < \frac{3 \epsilon_{\texttt{PT}}}{2}, \quad |\Prob{\pi \models \varphi} - \Prob{\pi^\ast \models \varphi}| \leq \epsilon_{\varphi}.
\end{equation*}
\end{prop}

\begin{proof}
Suppose $v_{\omega} < v_{\omega}'$ for any $\omega' \in  \Omega$, with $\omega \in \Omega$. Fix some $\omega'$. Denote the greedy policies $\pi_{v_{\omega}}, \pi_{v_{\omega'}}$ wrt $v_{\omega}, v_{\omega'}$. Suppose $\tilde{V}_{\pi_{v_{\omega'}}} < \tilde{V}_{\pi_{v_{\omega}}}$. Then an error was made and 
\begin{align*}
    0 \leq \tilde{V}^P_{\pi_{v_{\omega}}} - \tilde{V}^P_{\pi_{v_{\omega'}}} &\leq \tilde{V}^P_{\pi_{v_{\omega}}} - \tilde{V}^{\tilde{P}}_{\pi_{v_{\omega}}} + \tilde{V}^{\tilde{P}}_{\pi_{v_{\omega}}} - v_{\omega} + v_{\omega} - v_{\omega'} + v_{\omega'} - \tilde{V}^{\tilde{P}}_{\pi_{v_{\omega'}}} + \tilde{V}^{\tilde{P}}_{\pi_{v_{\omega'}}} - \tilde{V}^P_{\pi_{v_{\omega'}}} \\
    &\leq \frac{\epsilon_{\texttt{PT}}}{2} + \frac{\epsilon_{\texttt{PT}}}{2} + 0 + 0 + \frac{\epsilon_{\texttt{PT}}}{2}  \\
    &\leq \frac{3\epsilon_{\texttt{PT}}}{2}
\end{align*}
where the second line comes from grouping each pair of elements from the first line and applying the bounds found in proof of Proposition \ref{prop:plan_transient_formal}.

On the other hand, suppose $p + \epsilon_{\varphi}= \Prob{\pi_{v_{\omega}} \models \varphi} + \epsilon_{\varphi} < \Prob{\pi_{v_{\omega'}} \models \varphi} = p'$. The same proof as in Prop \ref{prop:select_V} applies to show that the probability of satisfaction remains close:
\begin{align*}
    0 \leq \tilde{V}_{\pi_{v_{\omega'}}} - \tilde{V}_{\pi_{v_{\omega}}}  &\leq \underbrace{p'(J_{\pi_{v_{\omega'}}} + \lambda \sum_{i=1}^k \frac{p'_i}{p'} \widehat{g}_{\tilde{\pi}_{A_i}})}_{ \leq J_{\pi} + \lambda c_{\max}} - \underbrace{ p (J_{\pi_{v_{\omega}}} + \lambda \sum_{i=1}^k \frac{p_i}{p} \widehat{g}_{\tilde{\pi}_{A_i}})}_{\geq 0} +  \underbrace{(p -p')}_{< -\epsilon_{\varphi}}\frac{\bar{V}}{\epsilon_{\varphi}} \\
    &< J_{\pi} + \lambda c_{\max} - \bar{V} \\
    &\leq 0
\end{align*}
showing that $|p - p'| < \epsilon_{\varphi}.$ 

In particular, since the choice of $\omega'$ was arbitrary, it holds for $\omega'$ achieving $\omega' = \min_{\omega \in \Omega} \min_{\pi \in \tilde{\Pi}(\omega)} V_{\pi}$. Therefore the previous bounds all hold for with $p'$ replaced with $p^\ast$ and $\tilde{V}^P_{\pi_{v_{\omega'}}}$ replaced with $\tilde{V}^P_{\pi^\ast}$.

It is clear we can think of this non-blocking subroutine as checking the different inputs to Algo \ref{algo:plan_transient}, which requires $\phi_{\texttt{PT}}(\omega) = \frac{ \epsilon_{\texttt{PT}} c_{\min} }{14 |\SState \setminus \omega| \bar{V}^2 (1+\frac{1}{\epsilon_{\varphi}})^2 }$, occuring when $n = \tilde{\mathcal{O}}((\frac{|\SState \setminus \omega| \bar{V}^2}{ \epsilon_{\texttt{PT}} \epsilon_{\varphi}^2 c_{\min}})^{2})$ samples per state-action pair have been collected. Taking the maximum over $\omega \in \Omega$, we have $n = \tilde{\mathcal{O}}((\frac{|\SState| \bar{V}^2}{ \epsilon_{\texttt{PT}} \epsilon_{\varphi}^2 c_{\min}})^{2})$ samples required for each state-action pair in $\SState \times \A$.
\end{proof}

\begin{rem}
Recall $\SState^\ast$ is set of accepting states in Product-MDP $\X$.
This subroutine appears to have an exponential runtime in $|\SState^\ast|$; $\Omega$ is at most $2^{|\SState^\ast|}$, which is not related to the typical PAC parameters. In general, $\Omega$ is modestly small. 
\end{rem}

\begin{rem}
While the runtime scales poorly with $|\SState^\ast|$, the sample complexity remains PAC.
\end{rem}

\begin{rem}
We believe it is possible to bring the runtime of the subroutine to be polynomial in $|\SState^\ast|$ by leveraging the MEC quotient structure (see \cite{modelchecking}), but leave that for future work.
\end{rem}

\newpage
\section{Experiments}\label{app:experiments}

\subsection{Environments and Details}\label{sec:app:env_details}

\begin{figure}[!htp]
\begin{center}
\begin{tabular}{cc}
\minipage{.4\textwidth}
  \centering
  \includegraphics[width=\linewidth]{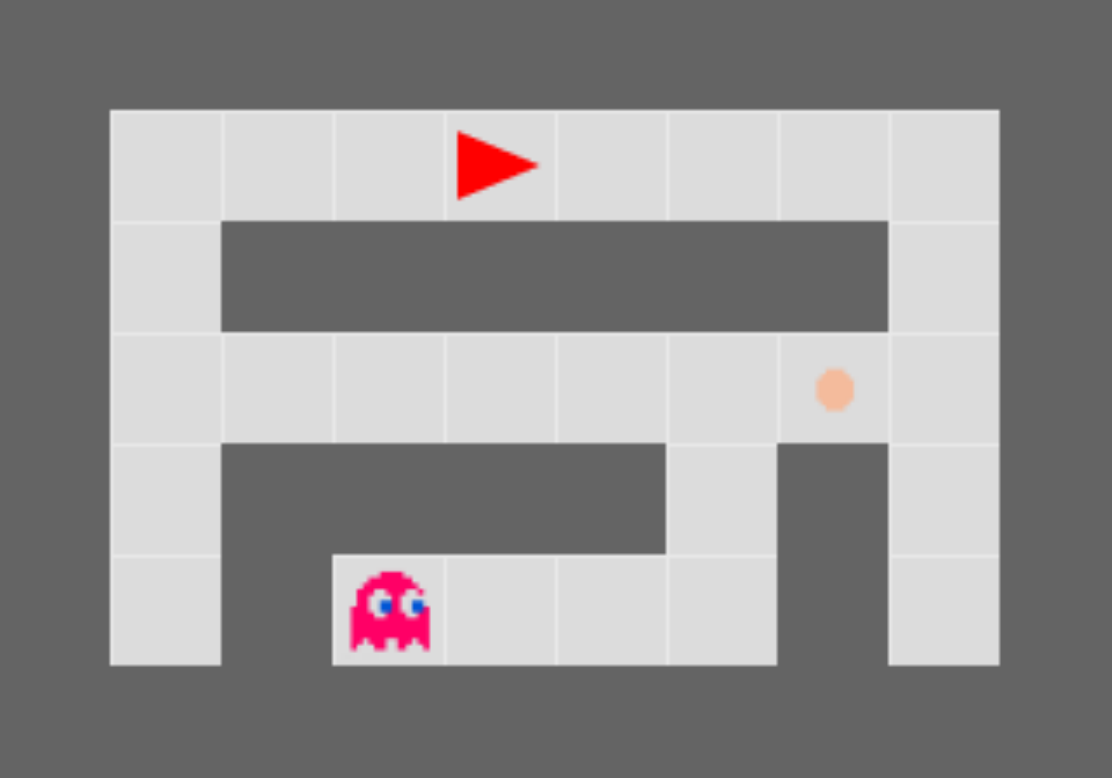}
\endminipage
&
\minipage{.4\textwidth}
  \includegraphics[width=\linewidth]{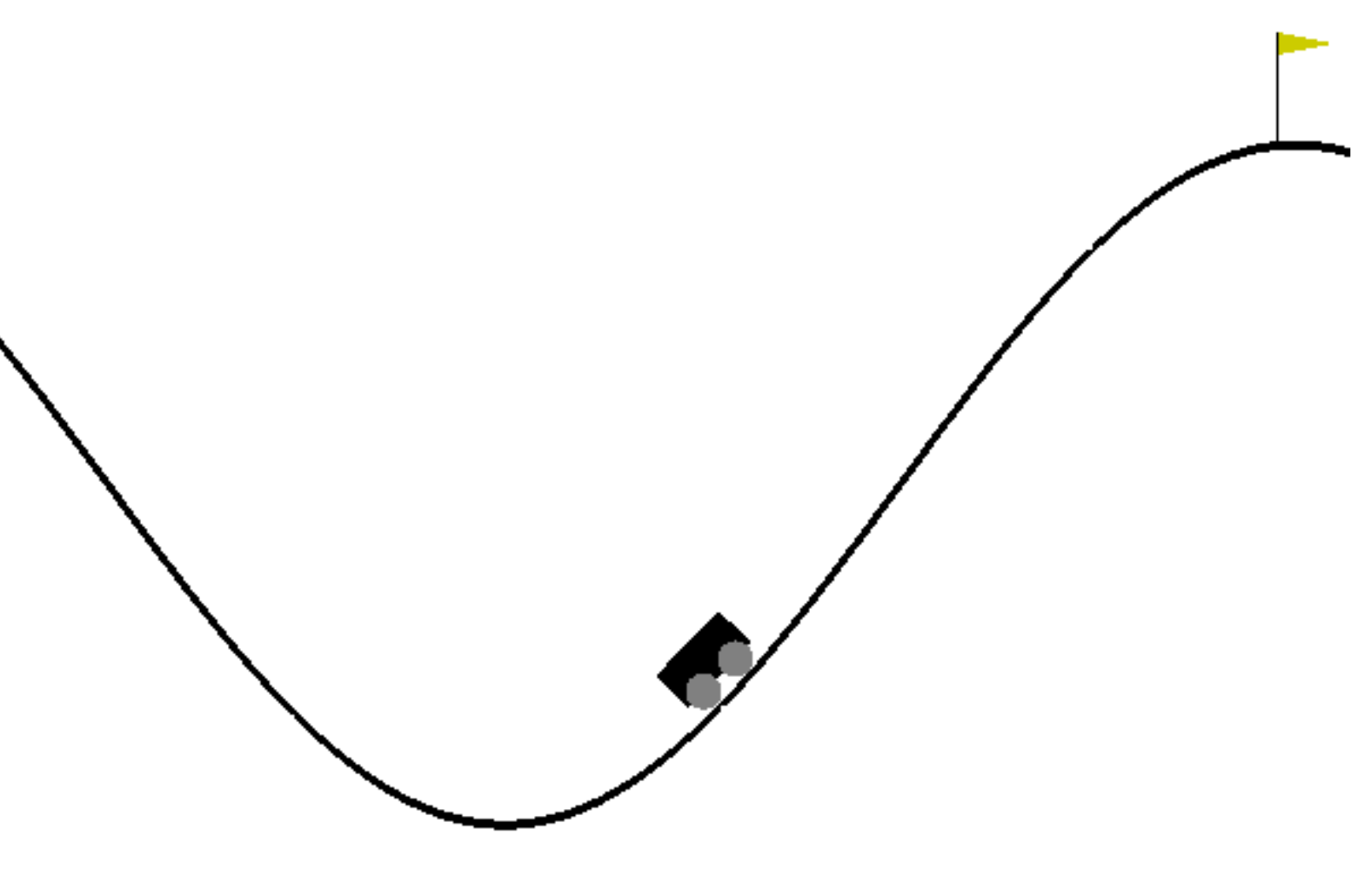}
\endminipage\hfill
\end{tabular}
\end{center}
\caption{\textit{Environment Illustrations. (Left) Pacman. $\varphi$ is for the agent, the red triangle, to eventually collect the food, given by the yellow dot, and always avoid the ghost, the red semicircle with eyes. (Right) Mountain Car (MC). $\varphi$ is to eventually reach the flag.
}}
 \label{fig:experiments_env_illustrations}
\end{figure}

\textbf{Pacman.} This environment (pictured in Fig \ref{fig:experiments_env_illustrations} Left) is a 5x8 gridworld. The starting positions of the agent (red triangle), food (yellow circle), and ghost (red semicircle with eyes), are as illustrated in Fig \ref{fig:experiments_env_illustrations}. The agent has $4$ cardinal directions at each state in addition to a ``do nothing'' action. The LTL specification is to eventually reach the food and to forever avoid the ghost ``F(food) \& G(!ghost)'', where the food state is labelled ``food'' and the ghost state is labelled ``ghost''. Once the food is picked up, it is gone. The ghost chases the agent (following the shortest path) with probability $.4$ and chooses a random action with probability $.6$. Though this is an infinite horizon problem, as there is no terminal state, we allow a maximum horizon of $H = 100$ in our experiments. We track how long the agent has avoided the ghost and whether the agent has picked up the food. To simplify verification, we say the agent has satisfied the spec if the food has been picked up and the ghost has been avoided for all $H$ timesteps. The cost function is defined as $1$ everywhere. 

For the shaped LCRL baseline, we use progression through the LDBA as a ``reward'': if the agent progresses to a new state in the automaton then the cost of that transition is $.1$ instead of $1$. The authors of LCRL used similar ideas in their code as well. However, we must note that progression-based cost shaping eliminates any guarantee of LTL satisfaction. An agent is incentivized to find cycles in the LDBA rather than find an accepting state. In the case when no such cycles exist, then this form of cost shaping can work.

\textbf{Mountain Car} This domain (pictured in Fig \ref{fig:experiments_env_illustrations} Right) is a discretization of the Mountain Car domain from OpenAI \cite{openAI}, with state-space given by tuple (position, velocity) and cost of $1$. We discretize the position space into $32$ equal size bins and the velocity into $32$ geometrically-spaced bins, allowing more granularity around low velocity than high velocity, making $32^2$ bins (states) in the MDP. The starting state is the standard MC starting state, but then placed in the appropriate bin. A bin can be converted back to (pos,vel), for purposes of sampling from $P$, by uniformly selecting from the valid positions/velocities implied by the bin. The agent has $3$ actions: accelerate left, do nothing, accelerate right. The specification is to eventually reach the goal state ``F(goal)'', the standard task, where any bin with position beyond the flag position is labelled ``goal''.

For the shaped LCRL baseline, we use a cost function of $c = .1$ if the change in position is positive and the agent accelerated right, likewise if the change is negative and the agent accelerated left, otherwise $c=1$. This cost function should incentivize the agent to seek actions which make the car go faster. Unlike the previous experiment, here cost-shaping has no effect on the guarantee of LTL satisfaction.

\textbf{Safe Delivery} This domain (pictured in Fig \ref{fig:motivation} Right) is a 4-state MDP: $(0)$ start state, $(1)$ sniffed packet, $(2)$ stolen packet $(3)$ delivered packet. In each state, the agent has two actions, $A$ and $B$. The transition function $P$ in the MDP is given by $P(0,A,1) = 1$, $P(0,B,2) = .5$, $P(0, B, 3) = .5$, $P(1,A,3)=1$, $P(1,B,3)=1$, $P(2,A,2)=1$, $P(2,B,2)=1$, $P(3,A,3)=1$, $P(3,B,3)=1$. In other words, choosing action $A$ in the initial state immediately leads to a sniffed packet, which subsequently leads to the packet being delivered by any action. Alternatively, choosing action $B$ in the initial state has a $50-50$ chance of having the packet stolen or immediately delivered, regardless of action. Once, stolen, it remains stolen. Once delivered, the packet remains delivered, regardless of action. The states are labelled as $L(0) = L(3) = $``safe''. The specification is to always stay in safe states: ``G(safe)''. Let all the costs be $1$. The Product-MDP can be seen in Figure \ref{fig:intuition} Right.

The probability-optimal and cost-optimal policy is then choosing $B$ is state $0$ and then arbitrarily afterward. The maximum probability of satisfying the policy is $50\%$ because $50\%$ of the time the packet gets stolen. Though this is an infinite horizon problem, as there is no terminal state, we allow a maximum horizon of $H = 100$ in our experiments. Thus, the average number of timesteps should be $.5*H = 50$.

Similarly to Pacman, for the shaped LCRL baseline, we use progression through the LDBA as a ``reward'': if the agent progresses to a new state in the automaton then the cost of that transition is $.5$ instead of $1$.

\textbf{Infinite Loop} This environment (pictured in Fig \ref{fig:motivation} Left) is a 2x5 gridworld. The agent starts in the bottom right corner. The agent has $4$ cardinal directions at each state in addition to a ``do nothing'' action. We consider two specifications:

$\varphi_1$: The LTL specification is to perpetually visit the office (in the top right corner) followed by the coffee room (top left corner): ``GF(o \& XFc)'', where the office is labelled $o$ and the coffee room is labelled $c$. The Product MDP is illustrated in Figure \ref{fig:intuition} Center.

$\varphi_2$: We require the agent to 
\begin{equation}\label{eq:varphi}
    \text{``G((c -> XXXXXo) \& (o ->XXXXXc)) \& Xo''},
\end{equation} meaning to get to first get to office in 1 step, then repeatedly reach the coffee room in 5 steps followed by the office in 5 steps.  

Similarly to Pacman, for the shaped LCRL baseline, we use progression through the LDBA as a ``reward'': if the agent progresses to a new state in the automaton then the cost of that transition is $.5$ instead of $1$.

\subsection{Hyperparameters}\label{sec:app:hyperparam}

We use the following hyperparameters for our experiments. Each set of hyperparameters was run with $20$ seeds, with the exception of Safe Delivery which was run with $40$ seeds.

\begin{figure}[!h]
\begin{minipage}[t]{1\textwidth}
\vspace{-.1in}
\captionof{table}{Hyperparameters}
\label{tab:hyperparam}
\vspace{-0.05in}
\begin{center}
\begin{small}
\begin{tabular}{cccccc}
\toprule
Param(s)   & Infinite Loop $\varphi_1$         & Infinite Loop $\varphi_2$                                   & Safe Delivery              & Pacman                        & MC                    \\
$\bar{V}$  &      $50$                &      $50$                                       &     $10$                   &      $100$                    &      $150$            \\
$c_{\min}$ &      $1$                 &      $1$                                        &     $1$                    &       $1$                     &       $1$             \\
$c_{\max}$ &      $1$                 &      $1$                                        &     $1$                    &       $1$                     &       $1$             \\
$\varphi$  & $\text{GF(o \& XFc)}$     & See $\varphi_2$ in \eqref{eq:varphi} &   $\text{G(!unsafe)}$                                                      &  $\text{F(food0) \& G!ghost}$ &  $\text{Fgoal}$      \\
$\epsilon$     & $3$                      & $3$                                             &      $3$                   &      $3$                      &      $10$             \\
$\delta$   & $.1$                     & $.1$                                            &      $.1$                  &      $.1$                     &      $.1$             \\
\hline \\
LCRL Params     & Infinite Loop          & Infinite Loop 2              & Safe Delivery                    & Pacman                 & MC                    \\
Max Traj len.   &      $100$          &         $100$                   &         $100$                     &         $100$        &         $200$        \\
$\gamma$        &      $.99$          &         $.99$                   &         $.99$                     &         $.99$        &         $.95$        \\
Learning rate   &      $.95$          &         $.95$                   &         $.95$                     &         $.95$        &         $.9$        \\
 \midrule
\end{tabular}
\end{small}
\end{center}

  \end{minipage}
  \end{figure}

\subsection{Additonal Results}

\begin{figure}[!htp]
\includegraphics[width=\linewidth]{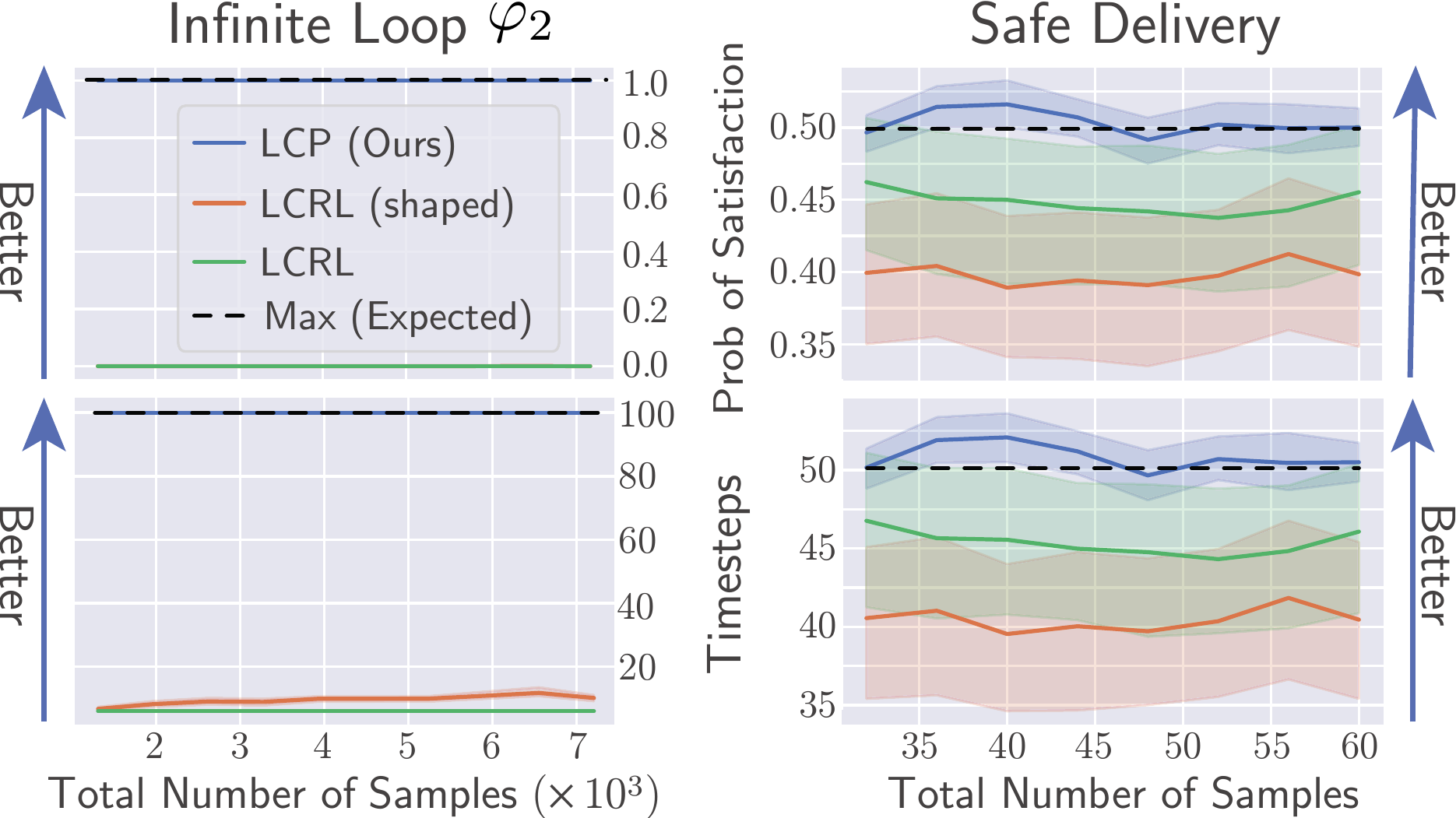}
\caption{\textit{Additional Results. (Left Column) Infinite Loop 2. $\varphi$ is a specific trajectory that needs to be followed: first get to the office in 1 timestep and then the coffee room in 5 and then back to the office in 5, over and over. (Right) Safe Delivery (Right Column). $\varphi$ is to always be safe.
}}
 \label{fig:experiments_additional_results}
\end{figure}

In this section we examine additional results for the experiments we ran. 

For the Infinite Loop environment under $\varphi_2$, we see (Figure \ref{fig:experiments_additional_results} Left Column) that our method is able to follow the trajectory specified by $\varphi_2$ even in low sample regimes. The learning signal for LCRL is very poor as the episode terminates extremely quickly if the agent does not get to the next location that it needs to be in within the allotted time. The shaped LCRL only does marginally better, but still struggles to satisfy the LTL with any probability.

For the Safe Delivery environment, we see (Figure \ref{fig:experiments_additional_results} Left Right) that our method picks out the probability-optimal policy. LCRL is nearly optimal. The sparsity of this problem is significantly less as the feedback for spec satisfaction verification comes after a single timestep. Interestingly, cost shaping in Safe Delivery performs worse than straight LCRL. This isn't surprising since, as noted, the verification feedback comes after a single timestep and is more important than any cost-shaping. However, cost-shaping muddles the feedback making shaped LCRL perform worse. We speculate that with only a few hundred or thousand more samples, both LCRL and shaped LCRL would reach the optimal policy. Recall that LCRL and shaped LCRL are not the same as Q-learning, as they operate in the product-MDP rather than the underlying MDP. Thus, these observations are still consistent with our Motivation section 
(Section \ref{sec:motivation}), insisting that Q-learning would have trouble in this environment.

\subsection{Policies}

\begin{figure}[!htp]
\begin{center}
\begin{tabular}{cc}
\minipage{.4\textwidth}
  \centering
  \includegraphics[width=\linewidth]{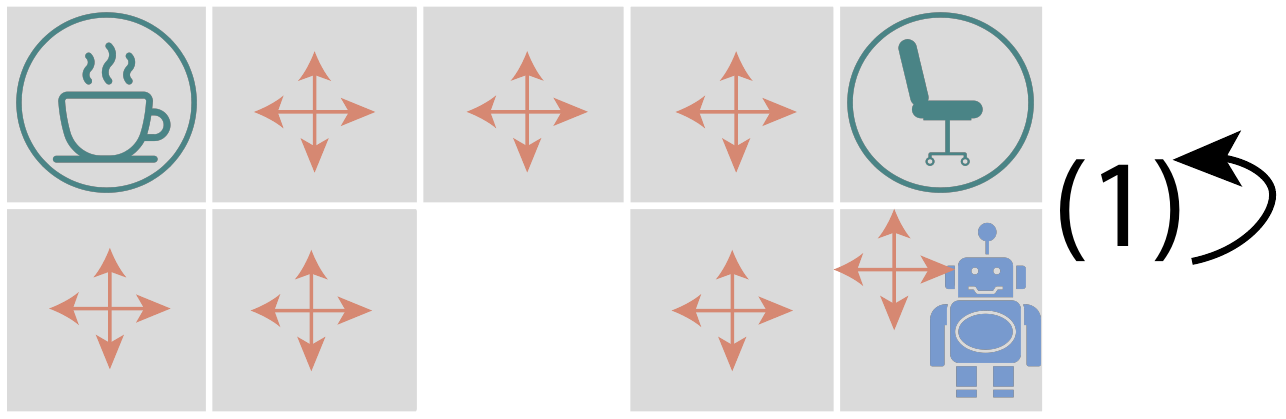}
\endminipage
&
\minipage{.4\textwidth}
  \includegraphics[width=\linewidth]{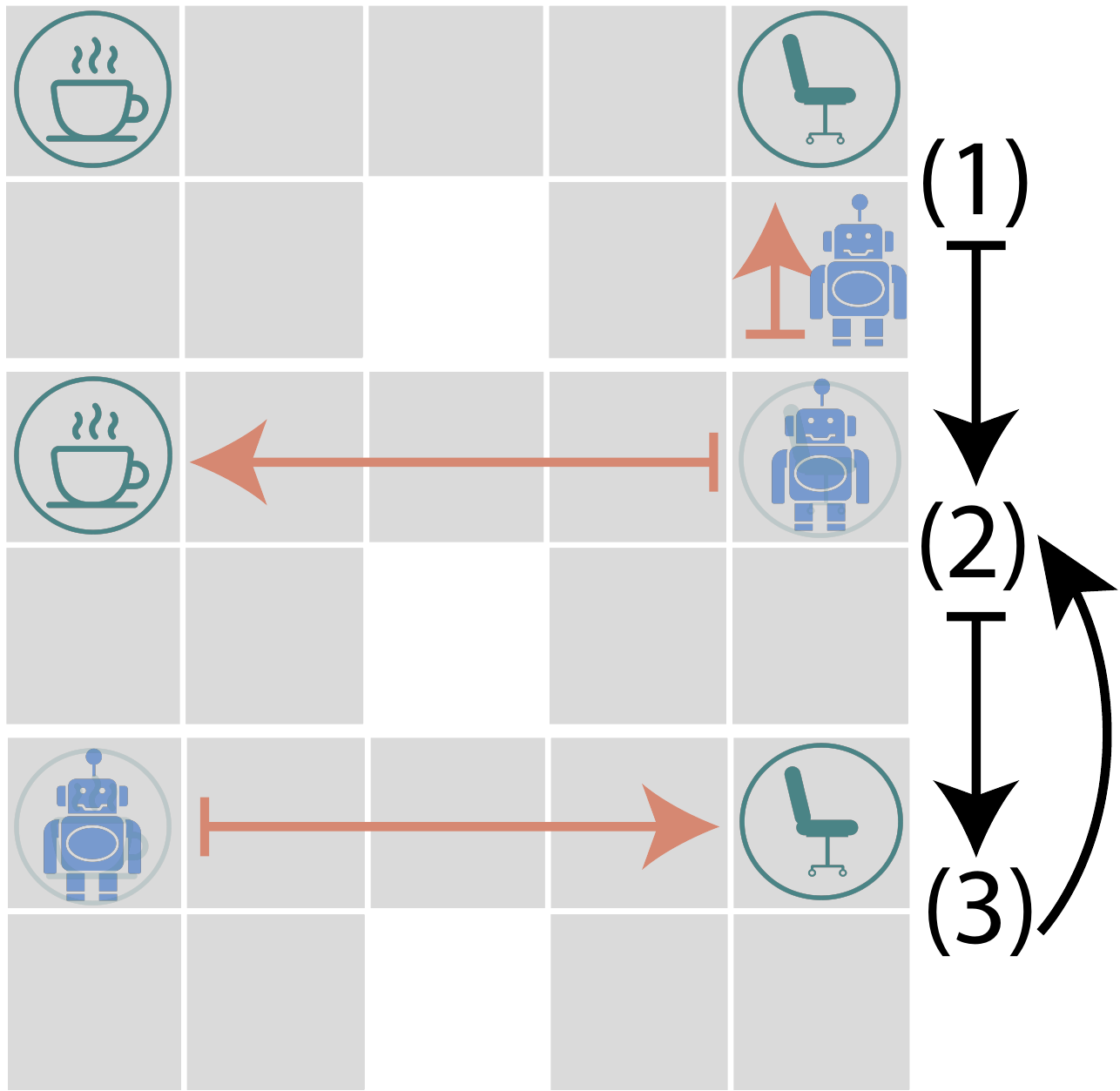}
\endminipage\hfill
\end{tabular}
\end{center}
\caption{\textit{Types of policies for different $\varphi$. (Left) Infinite Loop $\varphi_1$. $\varphi_1$ is to go perpetually walk between the office and the coffee room (Right) Infinite Loop $\varphi_2$. $\varphi_2$ is to get to the office in 1 time step then perpetually, take 5 timesteps to get to the coffee room and 5 steps back to the office.
}}
 \label{fig:experiments_policy_types}
\end{figure}

In this section we examine the policies induced by different specificity in specifications. In particular, we consider the Infinite Loop environment with two different specifications $\varphi_1, \varphi_2$, see Section \ref{sec:app:env_details}, \ref{sec:app:hyperparam} for a description. For $\varphi_1$, we only require that the agent ``eventually'' navigate between the office and coffee room. The agent is incentivized to stay in place (create a cost-1 cycle) for as long as possible and very infrequently take a random action. Of course, eventually taking random actions will loop the agent between the office and coffee room. This behavior is illustrated in Figure \ref{fig:experiments_policy_types} Left, where the agent is always in LDBA state 1 and takes random actions with low probability and does nothing with high probability. It takes exponential time for the agent to make a single loop between the office and coffee room. 

On the other hand, we may want the agent to move quickly. In this case, we can be more specific and use specification $\varphi_2$. The behavior for an agent satisfying $\varphi_2$ is illustrated in Figure \ref{fig:experiments_policy_types} Right. The agent gets to LDBA state 2 by first reaching the office in a single time step. Then the agent loops between LDBA states 2 and 3 by reaching the coffee room and office, repeatedly, within the allotted time. If the agent does not reach the office or chair within the allotted time, there is a fourth LDBA state (unpictured) which is a sink denoting failure of the spec. In essence, the LDBA has created the options, or hierarchy, of solving the problem, as noted in Section \ref{sec:motivation}. It takes $10$ timesteps for the agent to make a single loop between the office and coffee room.

Notice that the high level description of the task is unchanged, but the details of how the task is accomplished is much more specific in $\varphi_2$ rather than $\varphi_1$. This demonstrates that writing LTL task specifications is flexible, but requires thought about ``how'' the task should be accomplished. 



\end{document}